\documentclass{article}

\usepackage[final]{neurips_2022}

\usepackage[utf8]{inputenc} 
\usepackage[T1]{fontenc}    
\usepackage{hyperref}       
\usepackage{url}            
\usepackage{booktabs}       
\usepackage{amsfonts}       
\usepackage{nicefrac}       
\usepackage{microtype}      
\usepackage{xcolor}         

\usepackage{amsmath,amscd,amssymb,amsthm}
\usepackage{verbatim}
\usepackage{thmtools}
\usepackage{thm-restate}

\declaretheorem[name=Theorem]{Theorem}

\declaretheorem[name=Lemma, numberlike=Theorem]{Lemma}

\newcommand{\R}{\mathbb{R}}
\newcommand{\E}{\mathop{\mathbb{E}}}
\newcommand{\argmin}{\mathop{\text{argmin}}}
\newcommand{\argmax}{\mathop{\text{argmax}}}

\newcommand{\bw}{\mathbf{w}}
\newcommand{\bu}{\mathbf{u}}
\newcommand{\bx}{\mathbf{x}}

\newcommand{\bv}{\mathbf{v}}
\newcommand{\bg}{\mathbf{g}}
\newcommand{\sigmaunit}{\nu}
\newcommand{\cA}{\mathcal{A}}

\newcommand\numberthis{\addtocounter{equation}{1}\tag{\theequation}}

\newcommand{\power}{\mathfrak{p}}

\newcommand{\one}{{\mathbf 1}}
\DeclareMathOperator{\sign}{sign}
\usepackage{graphicx}
\usepackage{sidecap}
\sidecaptionvpos{figure}{c}

\usepackage{algorithm,algpseudocode}
\algnewcommand{\Initialize}[1]{%
  \State \textbf{Initialize:}
  \Statex \hspace*{\algorithmicindent}\parbox[t]{.6\linewidth}{\raggedright #1}
}

\usepackage{comment}

\usepackage{wrapfig}
\usepackage{floatflt}
\usepackage{times}
\allowdisplaybreaks
\title{Parameter-free Regret in High Probability with Heavy Tails}

\author{%
  Jiujia Zhang\\
  Electrical and Computer Engineering\\
  Boston University\\
  \texttt{jiujiaz@bu.edu}\\
  \And
  Ashok Cutkosky\\
  Electrical and Computer Engineering\\
  Boston University\\
  \texttt{ashok@cutkosky.com}
}

\begin{document}

\maketitle

\begin{abstract}
We present new algorithms for online convex optimization over unbounded domains that obtain parameter-free regret in high-probability given access only to potentially heavy-tailed subgradient estimates. Previous work in unbounded domains considers only in-expectation results for sub-exponential subgradients. Unlike in the bounded domain case, we cannot rely on straight-forward martingale concentration due to exponentially large iterates produced by the algorithm. We develop new regularization techniques to overcome these problems. Overall, with probability at most $\delta$, for all comparators $\bu$ our algorithm achieves regret $\tilde O( \|\bu\|T^{1/\power}\log(1/\delta))$ for subgradients with bounded $\power^{th}$ moments for some $\power \in (1, 2]$.
\end{abstract}

\section{Introduction}
In this paper, we consider the problem of online learning with convex losses, also called online convex optimization, with heavy-tailed stochastic subgradients. In the classical online convex optimization setting, given a convex set $\mathcal{W}$, a learning algorithm must repeatedly output a vector $\bw_t\in \mathcal{W}$, and then observe a convex loss function $\ell_t:\mathcal{W}\to \R$ and incur a loss of $\ell_t(\bw_t)$. After $T$ such rounds, the algorithm's quality is measured by the \emph{regret} with respect to a fixed competitor $\bu \in \mathcal{W}$:
\begin{align*}
    R_T(\bu) & = \sum_{t=1}^{T} \ell_t(\bw_t) - \sum_{t=1}^{T} \ell_t(\bu)
\end{align*}
Online convex optimization is widely applicable, and has been used to design popular stochastic optimization algorithms (\citep{duchi10adagrad, kingma2014adam,reddi2018on}), for control of linear dynamical systems \citep{agarwal2019online}, or even building concentration inequalities \citep{vovk2007hoeffding,waudby2020estimating,orabona2021tight}.

A popular approach to this problem reduces it to \emph{online linear optimization} (OLO): if $\bg_t$ is a subgradient of $\ell_t$ at $\bw_t$, then  $R_T(\bu)\le \sum_{t=1}^T 
\langle \bg_t,\bw_t-\bu\rangle$ so that it suffices to design an algorithm that considers only linear losses $\bw\mapsto  \langle \bg_t, \bw\rangle$. Then, by assuming that the domain $\mathcal{W}$ has some finite diameter $D$, standard arguments show that online gradient descent \citep{zinkevich2003online} and its variants achieve $R_T(\bu)\le O(D\sqrt{T})$ for all $\bu\in \mathcal{W}$. See the excellent books \cite{cesabianchi06prediction, shalev2011online, hazan2019introduction, orabona2019modern} for more detail.

Deviating from the classical setting, we study the more difficult case in which, (1) the domain $\mathcal{W}$ may have \emph{infinite} diameter (such as $\mathcal{W}=\R^d$), and (2) instead of observing the loss $\ell_t$, the algorithm is presented only with a potentially heavy-tailed stochastic subgradient estimate $\bg_t$ with $\E[\bg_t|\bw_t] \in \partial \ell_t(\bw_t)$. Our goal is to develop algorithms that, with high probability, obtain essentially the same regret bound that would be achievable even if the full information was available. 

Considering only the setting of infinite diameter $\mathcal{W}$ with \emph{exact} subgradients $\bg_t\in \partial \ell_t(\bw_t)$, past work has achieved bounds of the form $R_T(\bu)\le \tilde O(\epsilon+\|\bu\|\sqrt{T})$ for all $\bu\in \mathcal{W}$ simultaneously for any user-specified $\epsilon$, directly generalizing the $O(D\sqrt{T})$ rate available when $D<\infty$ \citep{orabona2016coin, cutkosky2018black, foster2017parameter, mhammedi2020lipschitz, chen2021impossible}. As such algorithms do not require knowledge of the norm $\|\bu\|$ that is usually used to specify a learning rate for gradient descent, we will call them \emph{parameter-free}. Note that such algorithms typically guarantee constant $R_T(0)$, which is not achieved by any known form of gradient descent.

While parameter-free algorithms appear to fully generalize the finite-diameter case, they fall short when $\bg_t$ is a stochastic subgradient estimate. In particular, lower-bounds suggest that parameter-free algorithms must require Lipschitz $\ell_t$ \citep{cutkosky2017online}, which means that care must be taken when using $\bg_t$ with unbounded noise as this may make $\ell_t$ ``appear'' to be non-Lipschitz. In the case of \emph{sub-exponential} $\bg_t$, \cite{jun2019parameter, van2019user} provide parameter-free algorithms that achieve $\E[R_T(\bu)]\le \tilde O(\epsilon + \|\bu\|\sqrt{T})$, but these techniques do not easily extend to heavy-tailed $\bg_t$ or to high-probability bounds. The high-probability statement is particularly elusive (even with sub-exponential $\bg_t$) because standard martingale concentration approaches appear to fail spectacularly. This failure may be counterintuitive: for \emph{finite diameter} $\mathcal{W}$, one can observe that $\langle \bg_t-\E[\bg_t], \bw_t-\bu\rangle$ forms a martingale difference sequence with variance determined by $\|\bw_t-\bu\|\le D$, which allows for relatively straightforward high-probability bounds. However, parameter-free algorithms typically exhibit \emph{exponentially growing $\|\bw_t\|$} in order to compete with all possible scales of $\|\bu\|$, which appears to stymie such arguments.

Our work overcomes these issues. Requiring only that $\bg_t$ have a bounded $\power^{th}$ moment for some $\power\in(1,2]$, we devise a new algorithm whose regret with probability at least $1-\delta$ is $R_T(\bu)\le \tilde O(\epsilon + \|\bu\|T^{1/\power}\log(1/\delta))$ for all $\bu$ simultaneously. The $T^{1/\power}$ dependency is unimprovable \cite{bubeck2013bandits, vural2022mirror}. Moreover, we achieve these results simply by adding novel and carefully designed regularizers to the losses $\ell_t$ in a way that converts any parameter-free algorithm with sufficiently small regret into one with the desired high probability guarantee.

\textbf{Motivation: } \textit{High-probability} analysis is appealing since it provides a confidence guarantee for an algorithm over a single run. This is crucially important in the online setting in which we must make irrevocable decisions. It is also important in the standard stochastic optimization setting encountered throughout machine learning as it ensures that even a single potentially very expensive training run will produce a good result. See \cite{harvey2019simple, li2020high, madden2020high, kavis2022high} for more discussion on the importance of high-probability bounds in this setting. This goal naturally synergizes with the overall objective of \textit{parameter-free} algorithms, which attempt to provide the best-tuned performance after a single pass over the data. In addition, we consider the presence of \textit{heavy-tailed} stochastic gradients, which are empirically observed in large neural network architectures \cite{zhang2020adaptive, zhou2020towards}. The online optimization problem we consider is actually fundamentally more difficult than the stochastic optimization problem: indeed \cite{carmon2022making} show that lower bounds for parameter-free online optimization to not apply to stochastic optimization, and provide a high-probability analysis for the latter setting. In contrast, the more flexible online setting allows us build more robust algorithms that can perform well in non-stationary or even adversarial environments.

\textbf{Contribution and Organization:} After formally introducing and discussing our setup in Sections~\ref{sec:prelim}, we then proceed to conduct an initial analysis for the 1-D case $\mathcal{W}=\R$ in \ref{sec:nontrivial}. First (Section~\ref{sec:sub_exp}), we introduce a parameter-free algorithm for \emph{sub-exponential} $g_t$ that achieves regret $\tilde{O}(\epsilon+|u|\sqrt{T})$ in high probability. This already improves significantly on prior work, and is accomplished by introducing a novel regularizer that ``cancels'' some unbounded martingale concentration terms, a technique that may have wider application. Secondly (Section~\ref{sec:heavy-tailed}), we extend to \textit{heavy-tailed} $g_t$ by employing clipping, which has been used in prior work on optimization \citep{bubeck2013bandits,gorbunov2020stochastic, zhang2020adaptive, cutkosky2021high} to convert heavy-tailed estimates into sub-exponential ones. This clipping introduces some bias that must be carefully offset by yet another novel regularization (which may again be of independent interest) in order to yield our final $\tilde O(\epsilon + |u|T^{1/\power})$ parameter-free regret guarantee. Finally (Section~\ref{sec:dim_free}), we extend to arbitrary dimensions via the reduction from \cite{cutkosky2018black}.

\section{Preliminaries}\label{sec:prelim}
Our algorithms interact with an adversary in which for $t=1\dots T$ the algorithm first outputs a vector $\bw_t\in \mathcal{W}$ for $\mathcal{W}$ a convex subset of some real Hilbert space, and then the adversary chooses a convex and $G$-Lipschitz loss function $\ell_t:\mathcal{W}\to \R$ and a distribution $P_t$ such that for $\bg_t \sim P_t$, $\E[\bg_t]\in \partial \ell_t(\bw_t)$ and $\E[\|\bg_t-\E[\bg_t]\|^\power]\le \sigma^\power$ for some $\power\in(1,2]$. The algorithm then observes a random sample $\bg_t\sim P_t$. After $t$ rounds, we compute the \emph{regret}, which is a function $R_t(\bu) = \sum_{i=1}^t \ell_i(\bw_i)-\ell_i(\bu)$. Our goal is to guarantee $R_T(\bu)\le \epsilon +\tilde O( \|\bu\|T^{1/\power})$ for all $\bu$ simultaneously with high probability.

Throughout this paper we will employ the notion of a \emph{sub-exponential} random sequence:
\begin{restatable}{Definition}{defsubexp}\label{def:subexp}
Suppose $\{X_t\}$ is a sequence of random variables adapted to a filtration $\mathcal{F}_t$ such that $\{X_t,\mathcal{F}_t\}$ is a martingale difference sequence. Further, suppose $\{\sigma_t,b_t\}$ are random variables such that $\sigma_t,b_t$ are both $\mathcal{F}_{t-1}$-measurable for all $t$. Then, $\{X_t,\mathcal{F}_t\}$ is $\{\sigma_t,b_t\}$ sub-exponential if
\[
\mathbb{E}[\exp(\lambda X_t)|\mathcal{F}_{t-1}]\le \exp(\lambda^2 \sigma_t^2/2)
\]
almost everywhere for all $\mathcal{F}_{t-1}$-measurable $\lambda$ satisfying $\lambda<1/b_t$.
\end{restatable}
We drop the subscript $t$ when we have uniform (not time-varying) sub-exponential parameters $(\sigma, b)$. 
We use bold font ($\bg_t$) to refer to vectors and normal font $(g_t)$ to refer to scalars. Occasionally, we abuse notation to write $\nabla \ell_t(\bw_t)$ for an arbitrary element of $\partial \ell_t(\bw_t)$. 

We present our results using $O(\cdot)$ to hide constant factors, and $\tilde{O}(\cdot)$ to hide $\log$ factors (such as some power of $\log T$ dependence) in the main text, the exact results are left at the last line of the proof for interested readers.

Finally, observe that by the unconstrained-to-constrained conversion of \cite{cutkosky2018black}, we need only consider the case that $\mathcal{W}$ is an entire vector space. By solving the problem for this case, the reduction implies a high-probability regret algorithm for any convex $\mathcal{W}$.

\section{Challenges}\label{sec:nontrivial}
A reader experienced with high probability bounds in online optimization may suspect that one could apply fairly standard approaches such as gradient clipping and martingale concentration to easily achieve high probability bounds with heavy tails. While such techniques do appear in our development, the story is far from straightforward. In this section, we will outline these non-intuitive difficulties. For a further discussion, see Section 3 of \cite{jun2019parameter}.

For simplicity, consider $w_t \in \R$. Before attempting a high probability bound, one may try to derive a regret bound in expectation with heavy-tailed (or even light-tailed) gradient $g_t$ via the following calculation:
\begin{align*}
   \E[R_T(u)] & =  \E \left[ \sum_{t=1}^{T} \ell_t(w_t) - \ell_t(u) \right] \leq \sum_{t=1}^{T} \E \left[ \langle g_t, w_t - u \rangle \right] +  \sum_{t=1}^{T} \E \left[ \langle \nabla \ell_t(w_t) - g_t, w_t- u \rangle \right]
\end{align*}
The second sum from above vanishes, so one is tempted to send $g_t$ directly to some existing parameter-free algorithm to obtain low regret. Unfortunately, most parameter-free algorithms require a uniform bound on $|g_t|$ - even a \emph{single} bound-violating $g_t$ could be catastrophic \citep{cutkosky2017online}. With heavy-tailed $g_t$, we are quite likely to encounter such a bound-violating $g_t$ for any reasonable uniform bound. In fact, the issue is difficult even for light-tailed $g_t$, as described in detail by \cite{jun2019parameter}.

A natural approach to overcome this uniform bound issue is to incorporate some form of clipping, a commonly used technique controlling for heavy-tailed subgradients. The clipped subgradient $\hat{g}_t$ is defined below with a positive clipping parameter $\tau$ as:
\begin{align*}
    \hat{g}_t = \frac{g_t}{|g_t|} \min (\tau, |g_t|)
\end{align*}
If we run algorithms on uniformly bounded $\hat{g}_t$ instead, the expected regret can now be written as:
\begin{align*}
   \E[R_T(u)] & \leq \underbrace{\sum_{t=1}^{T} \E \left[ \langle \hat{g}_t, w_t - u \rangle \right]}_{\text{parameter-free regret}} +   \underbrace{\sum_{t=1}^{T} \E \left[ \langle \E[\hat{g}_t] - \hat{g}_t, w_t - u \rangle \right]}_{\text{martingale concentration?}} + \underbrace{\sum_{t=1}^{T} \E \left[ \langle \nabla \ell_t(w_t) - \E[\hat{g}_t], w_t- u \rangle \right]}_{\text{bias}}\numberthis\label{eqn:difficultregret}
\end{align*}
Since $|\hat{g}_t| \leq \tau$, the first term can in fact be controlled for appropriate $\tau$ at a rate of $\tilde O(\epsilon + |u|\sqrt{T})$ using sufficiently advanced parameter-free algorithms (e.g. \cite{cutkosky2018black}). However, now bias accumulates in the last term, which is difficult to bound due to the dependency on $w_t$. On the surface, understanding this dependency appears to require detailed (and difficult) analysis of the dynamics of the parameter-free algorithm. In fact, from naive inspection of the updates for standard parameter-free algorithms, one expects that $|w_t|$ could actually grow exponentially fast in $t$, leading to a very large bias term.

Finally, disregarding these challenges faced even in expectation, to derive a high-probability bound the natural approach is to bound the middle sum in (\ref{eqn:difficultregret}) via some martingale concentration argument. Unfortunately, the variance process for this martingale depends on $w_t$ just like the bias term. In fact, this issue appears even if the original $g_t$ already have bounded norm, which is the most extreme version of \emph{light} tails! Thus, we again appear to encounter a need for small $w_t$, which may instead grow exponentially. In summary, the unbounded nature of $w_t$ makes dealing with any kind of stochasticity in the $g_t$ very difficult. In this work we will develop techniques based on regularization that intuitively force the $w_t$ to behave well, eventually enabling our high-probability regret bounds.

\section{Bounded Sub-exponential Noise via Cancellation}\label{sec:sub_exp}
In this section, we describe how to obtain regret bound in high probability for stochastic subgradients $g_t$ for which $\E[g_t^2]\le \sigma^2$ and $|g_t|\le b$ for some $\sigma$ and $b$ (in particular, $g_t$  exhibits $(\sigma,4b)$ sub-exponential noise). We focus on the 1-dimensional case with $\mathcal{W}=\R$. The extension to more general $\mathcal{W}$ is covered in Section~\ref{sec:dim_free}. Our method involves two coordinated techniques. First, we introduce a carefully designed regularizer $\psi_t$ such that \emph{any algorithm} that achieves low regret with respect to the losses $w\mapsto g_t w + \psi_t(w)$ will automatically ensure low regret with high probability on the original losses $\ell_t$. Unfortunately, $\psi_t$ is not Lipschitz and so it is still not obvious how to obtain low regret. We overcome this final issue by an ``implicit'' modification of the optimistic parameter-free algorithm of \cite{cutkosky2019combining}.  Our overall goal is a regret bound of $R_T(u)\le \tilde O(\epsilon + |u|(\sigma+G)\sqrt{T} + b|u|)$ for all $u$ with high probability. Note that with this bound, $b$ can be $O(\sqrt{T})$ before it becomes a significant factor in the regret.

Let us proceed to sketch the first (and most critical) part of this procedure: Define $\epsilon_t = \nabla \ell_t(w_t) - g_t$, so that $\epsilon_t$ captures the ``noise'' in the gradient estimate $g_t$. In this section, we assume that $\epsilon_t$ is $(\sigma,4b)$ sub-exponential for all $t$ for some given $\sigma,b$ and $|g_t|\le b$. Then we can write:
\begin{align*}
    R_T(u) &\leq \sum_{t=1}^{T} \langle \nabla \ell_t(w_t), w_t-u \rangle  = \sum_{t=1}^{T} \langle g_t, w_t - u\rangle + \sum_{t=1}^{T} \langle \epsilon_t, w_t \rangle - \sum_{t=1}^{T} \langle \epsilon_t,  u\rangle \\
    & \leq\sum_{t=1}^{T} \langle g_t, w_t - u\rangle + \underbrace{\left|\sum_{t=1}^{T} \epsilon_t w_t \right| + |u| \left|\sum_{t=1}^{T} \epsilon_t\right|}_{\text{``noise term'', }\textsc{Noise}} \numberthis \label{eqn:motivation1}
\end{align*}
Now, the natural strategy is to run an OLO algorithm $\cA$ on the observed $g_t$, which will obtain some regret $R^{\cA}_T(u)=\sum_{t=1}^{T} \langle g_t, w_t - u\rangle$, and then show that the remaining \textsc{Noise} terms are small. To this end, from sub-exponential martingale concentration, we might hope to show that with probability $1-\delta$, we have an identity similar to:
\begin{align*}
    \textsc{Noise} &\le \sigma \sqrt{\sum_{t=1}^T w_t^2\log(1/\delta)} + b\max_t |w_t| \log(1/\delta)+ |u|\sigma \sqrt{T\log(1/\delta)} + |u|b\log(1/\delta)
\end{align*}

The dependency of $|u|$ above appears to be relatively innocuous as it only contributes $\tilde O(|u|\sigma \sqrt{T} + |u|b)$ to the regret. The $w_t$-dependent term is more difficult as it involves a dependency on the algorithm $\cA$. This captures the complexity of our unbounded setting: in a \emph{bounded domain}, the situation is far simpler as we can uniformly bound $|w_t|\le D$, ideally leaving us with an $\tilde O(D\sqrt{T})$ bound overall. 

Unfortunately, in the unconstrained case, $|w_t|$ could grow exponentially ($|w_t|\sim 2^t)$ even when $u$ is very small, so we cannot rely on a uniform bound. In fact, even in the finite-diameter case, if we wish to guarantee $R_T(0)\le \epsilon$, the bound $|w_t|\le D$ is still too coarse. The resolution is to instead feed the algorithm $\cA$ a \emph{regularized} loss $\hat{\ell}_t(w) = \langle g_t, w \rangle + \psi_t(w)$, where $\psi_t$ will ``cancel'' the $w_t$ dependency in the martingale concentration. That is, we now define $R^{\cA}_T(u)=\sum_{t=1}^T \hat \ell_t(w_t)-\hat\ell_t(u)$ and rearrange:
\begin{align*}
    \sum_{t=1}^{T} \langle g_t, w_t-u \rangle & \leq R_T^{\mathcal{A}}(u) - \sum_{t=1}^{T} \psi_t(w_t) + \sum_{t=1}^{T} \psi_t(u) \numberthis \label{eqn:motivation2}
\end{align*}
And now combine equations (\ref{eqn:motivation1}) and (\ref{eqn:motivation2}):
\begin{align*}
    R_T(u) &\leq R_T^{\mathcal{A}}(u) - \sum_{t=1}^{T} \psi_t(w_t)+ \sum_{t=1}^{T} \psi_t(u) + \textsc{Noise}\\
    &\leq R_T^{\mathcal{A}}(u) + \sigma \sqrt{\sum_{t=1}^T w_t^2\log(1/\delta)} + b\max_t |w_t| \log(1/\delta) - \sum_{t=1}^{T} \psi_t(w_t)\\
    &\qquad + |u|\sigma \sqrt{T\log(1/\delta)} + |u|b\log(1/\delta)+\sum_{t=1}^{T} \psi_t(u) \numberthis \label{eqn:motivation3}
\end{align*}
From this, we  can read off the desired properties of $\psi_t$: (1) $\psi_t$ should be large enough that $\sum_{t=1}^T \psi_t(w_t)\ge \sigma \sqrt{\sum_{t=1}^T w_t^2\log(1/\delta)} + b\max_t |w_t| \log(1/\delta)$, (2) $\psi_t$ should be small enough that $\sum_{t=1}^T \psi_t(u)\le\tilde O( |u|\sqrt{T})$, and (3) $\psi_t$ should be such that $R_T^{\cA}(u)=\tilde O(\epsilon + |u|\sqrt{T})$ for an appropriate algorithm $\cA$. If we can exhibit a $\psi_t$ satisfying all three properties, we will have developed a regret bound of $\tilde O(\epsilon+|u|\sqrt{T})$ in high probability.

It turns out that the modified Huber loss $r_t(w)$ defined in equation (\ref{eqn:huber_loss}) and (\ref{eqn:reg_with_number}) with appropriately chosen constants $c_1, c_2, p_1, p_2, \alpha_1, \alpha_2$ satisfies criterion (1) and (2). 
\begin{align*}
    r_t(w; c,p,\alpha_0) & =  
    \begin{cases}
    c \left(p |w| - (p-1) |w_t| \right) \frac{|w_t|^{p - 1}}{( \sum_{i=1}^{t} |w_i|^{p} + \alpha_0^p )^{1- 1/p} }, & |w| > |w_t|\\
    c |w|^{p} \frac{1}{( \sum_{i=1}^{t} |w_i|^{p} + \alpha_0^p )^{1-1/p} }, & |w| \leq |w_t|
    \end{cases} \numberthis \label{eqn:huber_loss}\\
    \psi_t(w) & = r_t(w; c_1,p_1,\alpha_1)+ r_t(w; c_2,p_2,\alpha_2) \numberthis \label{eqn:reg_with_number}
\end{align*}
Let us take a moment to gain some intuition for these functions $r_t$ and $\psi_t$. First, observe that $r_t$ is always continuously differentiable, and that $r_t$'s definition requires knowledge of $w_t$. This is acceptable because online learning algorithms must be able to handle even adaptively chosen losses. In particular, consider the $p=2$ case, $r_t(w;c,2,\alpha)$ for some positive constants $c$ and $\alpha$. We plot this function in Figure~\ref{fig:rplot}, where one can see that $r_t$ grows quadratically for $|w|\le |w_t|$,  but  grows only linearly afterwards so that $r_t$ is Lipschitz.

\begin{SCfigure}[][h]
  \centering
  \includegraphics[width=0.55\textwidth]{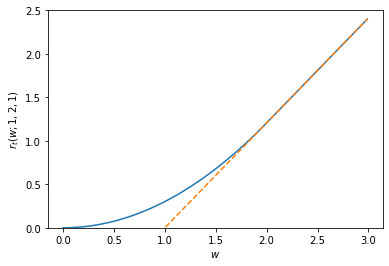}
  \caption{$r_t(w;1,2,1)$ when $\sum_{i=1}^t w_i^2=10$ and $w_t=2$. The dashed line has slope $cp\tfrac{|w_t|^{p-1}}{\left(\sum_{i=1}^t |w_i|^p + \alpha_0^p\right)^{1/p}}$, so that $r_t$ is quadratic for $|w|\le |w_t|$ and linear otherwise. Notice that $w_t$ is a constant used to define $r_t$ - it is not the argument of the function.}
  \label{fig:rplot}
\end{SCfigure}


Eventually, in Lemma~\ref{lemma:reg_ub_lb} we will show that this functions satisfies 
\begin{align*}
    &\sum_{t=1}^T r_t(w_t; c,2,\alpha)\ge c\sqrt{\sum_{t=1}^T w_t^2}-\alpha, &
    \sum_{t=1}^T r_t(u:c,2,\alpha)\le \tilde O(u\sqrt{T})
\end{align*}
so that for appropriate choice of $c$ and $\alpha$, $r_t(w;c,2,\alpha)$ will cancel the $O(\sqrt{\sum_{t=1}^T w_t^2})$ martingale concentration term while not adding too much to the regret - it satisfies criteria (1) and (2). The lower-bound follows from the standard inequality $\sqrt{a+b}\le \sqrt{a} + \frac{b}{\sqrt{a+b}}$ since $r_t(w_t) = c\frac{w_t^2}{\sqrt{\alpha^2+\sum_{i=1}^t w_i^2}}$. The upper-bound is more subtle, and involves the piece-wise definition. For simplicity, suppose it were true that either $|w_t| < |u|$ for all $t$ or $|w_t| \geq |u| $ for all $t$. In the former case, $\sum_{t=1}^T r_t(u)=O\left(|u|\sum_{t=1}^T \frac{|w_t|}{\sqrt{\alpha+\sum_{i=1}^t w_i^2}}\right)$, which via algebraic manipulation can be bounded as $\tilde O(|u|\sqrt{T})$. In the latter case, we have $\sum_{t=1}^T \frac{|u|^2}{\sqrt{\alpha^2+\sum_{i=1}^t w_t^2}}\le \sum_{t=1}^T \frac{|u|^2}{\sqrt{\alpha^2+tu^2}} = \tilde O(|u|\sqrt{T})$ so that both cases result in the desired bound on $\sum_{t=1}^T r_t(u)$. The general setting is handled by partitioning the sum into two sets depending on whether $|w_t|\le |u|$. In order to cancel the $\max_t |w_t|$ term in the martingale concentration, we employ $p=\log T$. This choice is motivated by the observation that $\|\bv\|_{\log T}\in[\|\bv\|_\infty, \exp(1)\|\bv\|_\infty]$ for all $\bv\in \R^{\log T}$. With this identity in hand, the argument is very similar to the $p=2$ case.

The correct values for the constants are provided in Theorem~\ref{thm:main_sub_exp}. Again, at a high level, the important constants are $p_1$ and $p_2$. With $p_1=2$, we allow $\sum_t r_t(w_t;p=2)$ to cancel out the $\sqrt{\sum_t w_t^2}$ martingale concentration term, while with $p_2=\log T$, $\sum_t r_t(w_t;p=\log T)$ cancels that $\max_t |w_t|$ term.

It remains to show that $\psi_t$ also allows for small $R_T^{\mathcal{A}}(u)$ and so satisfies criterion (3). Unfortunately, our setting for $c_2$ in the definition of $\psi_t$ is $\tilde O(b)$, which means that $\psi_t$ is $\tilde O(b)$-Lipschitz. Since we wish to allow for $b=\Theta(\sqrt{T})$, this means that we cannot simply let $\cA$ linearize $\psi_t$ and apply an arbitrary OLO algorithm. Instead, we must exploit the fact that $\psi_t$ is known \emph{before} $g_t$ is revealed. That is, algorithm $\mathcal{A}$ is chosen to exploit the structure composite loss $\hat{\ell}_t(w)$. Intuitively, the regret of a composite loss should depend only on the non-composite $g_t$ terms (as in e.g. \cite{duchi2010composite}). Our situation is slightly more complicated as $\psi_t$ depends on $w_t$ as well, but we nevertheless achieve the desired result via a modification of the parameter-free optimistic reduction in \cite{cutkosky2019combining}. 
For technical reasons, this algorithm still requires $|g_t|\le b$ with probability 1, but obtains regret only $R^{\cA}_T(u)\le \tilde O(\epsilon + |u|\sigma\sqrt{T} +|u|b)$. This technical limitation is lifted in the following section. 

\begin{algorithm}[H]
\caption{Sub-exponential Noisy Gradients with Optimistic Online Learning}\label{alg:sub_exp_combine}
\begin{algorithmic}[1]
\Require $E[g_t] = \nabla \ell_t(w_t)$, $|g_t| \leq b, \E[g_t | w_t] \leq \sigma^2$ almost surely. Two online learning algorithms (e.g. copies of Algorithm 1 from \cite{cutkosky2018black}) labelled as $\mathcal{A}_1, \mathcal{A}_2$ with domains $\R$ and $\R_{\ge 0}$ respectively. Time horizon $T$, $0 < \delta \leq 1$.
\Initialize{Constants $ \{c_1, c_2, p_1, p_2, \alpha_1, \alpha_2\}$ from Theorem~\ref{thm:main_sub_exp}.\\ 
$H=c_1p_1 + c_2 p_2$} \Comment{for defining $\psi_t$ in equation (\ref{eqn:reg_with_number})}
\For{$t=1$ to $T$}
    \State Receive $x_{t}'$ from $\mathcal{A}_1$, $y_t' $ from $\mathcal{A}_2$ 
    \State Rescale $x_t =  x_t'/(b+H) $, $y_t =y_t'/( H(b+H) )$
    \State Solve for $w_t$: $w_t = x_t - y_t \nabla \psi_t(w_t)$ \Comment{The solution exists by Lemma \ref{lemma:existence_sol}}
    \State Play $w_t$ to, suffer loss $\ell_t(w_t)$
    \State Receive $g_t$  with $\E[g_t] \in \partial\ell_t(w_t)$ 
    \State Compute $\psi_t(w) = r_t(w;c_1,p_1,\alpha_1)+r_t(w;c_2,p_2,\alpha_2)$ and $\nabla \psi_t(w_t)$ \Comment{equations (\ref{eqn:huber_loss}), (\ref{eqn:reg_with_number})}
    \State Send $(g_t + \nabla \psi_t(w_t))/(b+H)$ to $\mathcal{A}_1$ 
    \State Send $-\langle {g}_t + \nabla \psi_t(w_t), \nabla \psi_t(w_t) \rangle/H(b+H) $ to $\mathcal{A}_2$
\EndFor
\end{algorithmic}
\end{algorithm}

We display the method as Algorithm \ref{alg:sub_exp_combine}, which provides a regularization that cancels the $|w_t|$ dependent part of the \textsc{Noise} term in (\ref{eqn:motivation_heavy_1}). It also allows us to control $R_T^{\cA}(u)$ to order $\tilde O(\epsilon + |u|\sigma\sqrt{T} + b|u|)$ by taking account into the predictable structure of regularizer $\psi_t(w)$. The algorithm requires black-box access to two base online learning algorithms, which we denote $\cA_1$ and $\cA_2$ with domains $(-\infty,\infty)$ and $[0,\infty)$ respectively. These can be any algorithms that obtain so-called ``second-order'' parameter-free regret bounds, such as available in \cite{cutkosky2018black, van2019user, kempka2019adaptive, mhammedi2020lipschitz}. Roughly speaking, the role of $\cA_1$ is to provide an initial candidate ouput $x_t$ that is then ``corrected'' by $\cA_2$ using the regularization to obtain the final $w_t$.

Following the intuition previously outlined in this section, We first provide a deterministic regret guarantee on the quantity $R_T^{\mathcal{A}}(u) = \sum_{t=1}^T \hat \ell_t(w_t)-\hat\ell_t(u)$ as an intermediate result (Theorem~\ref{thm:base_algo}). Then, we provide the analysis of the full procedure of Algorithm~\ref{alg:sub_exp_combine} for the final high probability result (Theorem~\ref{thm:main_sub_exp}). Missing proofs are provided in the Appendix \ref{sec:app_optmistic} and \ref{sec:app_bernstein}.

\begin{restatable}{Theorem}{basealgo}
\label{thm:base_algo}
    Suppose $\mathcal{A}_1$  ensure that given some $\epsilon>0$ and a sequence  $c_t$ with $|c_t| \leq 1$:
    \begin{align*}
        \sum_{t=1}^{T} \langle c_t, w_t - u \rangle \leq \epsilon + A |u| \sqrt{\sum_{t=1}^{T} |c_t|^2 \left( 1 + \log \left( \frac{|u|^2 T^C}{\epsilon^2} + 1 \right) \right)} + B|u| \log \left(\frac{|u|T^C}{\epsilon} + 1 \right)
    \end{align*}
    for all $u$ for some positive constants $A, B, C$, and that $\mathcal{A}_2$ obtains the same guarantee for all $u\ge 0$, 
    then for $|g_t| \leq b $, $| \nabla \psi_t(w_t) | \leq H$, we have the following guarantee from Algorithm \ref{alg:sub_exp_combine},
    \begin{align*}
         R_T^{\mathcal{A}}(u)& \leq O\left[ \epsilon + |u| \left(\sqrt{\max \left(0, \sum_{t=1}^{T} | g_t |^2  - | \nabla \psi_t(w_t) |^2 \right) } + (b+H) \log T \right)\right]
    \end{align*}
\end{restatable}
Although this Theorem~\ref{thm:base_algo} is rather technical, the overall message is not too complicated. If we ignore the negative $|\nabla \psi_t(w_t)|^2$ terms, the bound simply says that the regret on the ``composite'' loss $\langle g_t,w\rangle + \psi_t(w)$ only increases with the apriori-unknown $g_t$, and \emph{not} with $\nabla \psi_t(w_t)$. With this result, we can formalize the intuition in this section to  provide the following high probability regret bound:

\begin{restatable}{Theorem}{mainsubexp}
\label{thm:main_sub_exp}
    Suppose $\{g_t\}$ are stochastic subgradients such that  $\E[g_t]\in \partial \ell_t(w_t)$, $|g_t| \leq b$ and $\E[g_t^2 | w_t] \leq \sigma^2$ almost surely for all $t$. Set the following constants for $\psi_t(w)$ shown in equation (\ref{eqn:reg_with_number}) for any $0 < \delta \leq 1$, $\epsilon > 0$,
    \begin{align*}
        \begin{matrix}
            c_1 = 2 \sigma \sqrt{ \log\left(\frac{32}{\delta}\left[\log\left( 2^{T+1} \right)+2\right]^2\right)},  & c_2 =  32b \log\left(\frac{224}{\delta}\left[\log\left( 1 + \frac{b}{\sigma} 2^{T+2}\right)+2\right]^2\right),\\
         p_1 = 2, ~~~~~~~~ p_2 = \log T, &\alpha_{1} = \epsilon/c_1, ~~~~~~~ \alpha_{2} = \epsilon \sigma/(4b(b+H)) \\
        \end{matrix}
    \end{align*}
    where $H = c_1p_1+c_2p_2$, $| \nabla \psi_t(w_t) | \leq  H$. Then, with probability at least $1-\delta$, algorithm \ref{alg:sub_exp_combine} guarantees 
\begin{align*}
    R_T(u) & \leq \tilde{O}\left[\epsilon \log \frac{1}{\delta} + |u|b \log\frac{1}{\delta} + |u|\sigma \sqrt{T\log \frac{1}{\delta}}\right]
\end{align*}

\end{restatable}

Note that this result is \emph{already} of interest: prior work on parameter-free algorithms with sub-exponential noise only achieve in-expectation rather than high probability results. Of course, there is a caveat: our bound requires that $|g_t|$ be uniformly bounded by $b$. Even though $b$ could be as large as $\sqrt{T}$, this is still a mild restriction. In the next section, we remove both this restriction as well as the light tail assumption all together.

\section{Heavy tails via Truncation}\label{sec:heavy-tailed}
In this section, we aim to give a high probability bound for heavy-tailed stochastic gradients $\bg_t$. Our approach builds on Section~\ref{sec:sub_exp} by incorporating gradient clipping with a clipping parameter $\tau \in \R^{+}$. 
\begin{align*}
    \hat{\bg}_t = \frac{ \bg_t}{\|\bg_t\|} \min(\tau, \|\bg_t\|)
\end{align*}
We continue to consider a 1-dimensional problem in this section, replacing the norm $\| \cdot \|$ with absolute value $|\cdot|$ and $\bg_t$ with $g_t$. The key insight is that the clipped $\hat{g}_t$ satisfies $\E[\hat{g}_t^2] \le 2^{\power-1}\tau^{2-\power}(\sigma^\power + G^\power)$ and of course $|\hat g_t|\le \tau$. Hence, a high probability bound could be obtained by feeding $\hat{g}_t$ into Algorithm~\ref{alg:sub_exp_combine} from Section~\ref{sec:sub_exp}. Let us formally quantify the effect of this clipping:
\begin{align*}
   R_T(u) & \leq \sum_{t=1}^{T} \langle \nabla \ell_t(w_t), w_t - u \rangle = \underbrace{\sum_{t=1}^{T} \langle \nabla \ell_t(w_t)- \E[\hat{g}_t], w_t - u \rangle}_{\text{bias}} + \underbrace{\sum_{t=1}^{T} \langle \E[\hat{g}_t], w_t - u \rangle}_{\text{Section~\ref{sec:sub_exp}}} \numberthis \label{eqn:motivation_heavy_1}
\end{align*}
Without clipping, we would have $\E[\hat{g}_t] = \nabla \ell_t(w_t)$, and so if we were satisfied with an in-expectation result, the first sum above would vanish. However, with clipping, the first sum actually represents some ``bias'' that must be controlled even to obtain an in-expectation result, let alone high probability. We control this bias using a cancellation-by-regularization strategy analogous at a high level to the one developed in Section~\ref{sec:sub_exp}, although technically quite distinct. After dealing with the bias, we must handle the second sum. Fortunately, since $\hat g_t$ is sub-exponential, bounding the second sum in high probability is precisely the problem solved in Section~\ref{sec:sub_exp}. We introduce the analysis in two elementary steps. For the purpose of bias cancellation, we define a linearized loss $\tilde{\ell}_t(w)$ with 
regularization function $\phi(w)$ 
\begin{align*}
    &\tilde{\ell}_t(w) = \langle \E[\hat{g}_t], w \rangle + \phi(w),  & \phi(w) =  2^{\power-1}(\sigma^{\power} + G^{\power}) |w| /\tau^{\power-1} \numberthis \label{eqn:lin_loss_heavy_tail} 
\end{align*}
the regret in equation (\ref{eqn:motivation_heavy_1}) can be re-written as
\begin{align*}
    & = \underbrace{ \sum_{t=1}^{T} \left(\langle \nabla \ell_t(w_t) - \E [ \hat{g}_t], w_t - u \rangle - \phi(w_t) + \phi(u)\right)}_{\text{bias cancellation}} + \underbrace{\sum_{t=1}^{T} \tilde{\ell}_t(w_t) - \tilde{\ell}_t(u)}_{\text{Section~\ref{sec:sub_exp}}} \numberthis \label{eqn:motivation_heavy_2}
\end{align*}
We will be able to show that the $w_t$-dependent terms of the first summation sum to a negative number and so can be dropped. This leaves only the $u$-dependent terms, which for appropriate choice of $\tau$ will be $\tilde O(|u|T^{1/\power})$.

Note that at this point, if we were satisfied with an \emph{in expectation} bound for heavy-tailed subgradient estimates (which would already be an interesting new result), we would not require the techniques of Section~\ref{sec:sub_exp}: we could instead define $\hat \ell_t(w) = \langle \hat g_t,w\rangle + \psi(w)$, so that the last sum is equal to $\sum_{t=1}^T \hat \ell_t(w_t) - \hat \ell_t(u)$ in expectation. Then, since $|\hat \ell_t(w_t)|\le O(\tau)$ with probability 1, we can control $\sum_{t=1}^T \hat \ell_t(w_t) - \hat \ell_t(u)$ using a parameter-free algorithm obtaining regret $\tilde O(|u|\sqrt{\sum_{t=1}^T |\nabla \hat \ell_t(w_t)|^2} + \tau |u|)$ to bound the total expected regret, yielding a simple way to recover prior work on expected regret with sub-exponential subgradients (up to logs), while extending the results to heavy-tailed subgradients.

However, since we \emph{do} aim for a high probability bound, we need to be more careful with the second summation. Fortunately, given that $\hat{g}_t$ is sub-exponential and bounded, and $\nabla \phi(w_t)$ is deterministic, we can supply $\hat{g}_t + \nabla \phi (w_t)$ to Algorithm~\ref{alg:sub_exp_combine} and then bound the sum in high probability by Theorem~\ref{thm:main_sub_exp}. We formalize the procedure as Algorithm~\ref{alg:heavy_tail_algo}, and its guarantee is stated in Theorem~\ref{thm:main_heavy}. The exact regret guarantee (including constants) can be found in Appendix \ref{sec:app_clipping}.

\begin{algorithm}[H]
\caption{Gradient clipping for $(\sigma, G)-$Heavy tailed gradients}\label{alg:heavy_tail_algo}
\begin{algorithmic}[1]
\Require $\E[g_t] = \nabla \ell_t(w_t)$, $|\E[g_t] | \leq G$, $\E[|g_t - \E[g_t] |^{\power}] \leq \sigma^{\power}$ for some $\power \in (1,2]$, Time horizon $T$, gradient clipping parameter $\tau$.
\State Initialize Algorithm~\ref{alg:sub_exp_combine} using the parameters of Theorem~\ref{thm:main_sub_exp}.
\For{$t=1$ to $T$}
    \State Receive $w_{t}$ from Algorithm~\ref{alg:sub_exp_combine}.
    \State Suffer loss $\ell_t(w_t)$, receive $g_t$
    \State Truncate $\hat{g}_t = \frac{g_t}{|g_t|} \min(\tau, |g_t|)$.
    \State Compute $\tilde g_t = \hat g_t + \nabla \phi_t(w_t)$ \Comment{$\phi(w)$ is defined in (\ref{eqn:lin_loss_heavy_tail}), $\E[\tilde g_t]\in \partial \tilde{\ell}_t(w_t) $.}
    \State Send $\tilde{g}_t$ to Algorithm~\ref{alg:sub_exp_combine} as $t^{th}$ subgradient.
\EndFor
\end{algorithmic}
\end{algorithm}

\begin{restatable}{Theorem}{mainheavy}
\label{thm:main_heavy}
    Suppose $\{g_t\}$ are heavy-tailed stochastic gradient such that $\E[g_t]\in \partial \ell_t(w_t)$, $|\E[g_t]| \leq G$, $\E[|g_t - \E[g_t]|^\power] \leq \sigma^{\power}$ for some $\power \in (1,2]$.
    If we set $\tau = T^{1/\power} (\sigma^{\power}+G^{\power})^{1/\power}$ 
then with probability at least $1-\delta$, Algorithm \ref{alg:heavy_tail_algo} guarantees:
\begin{align*}
    R_T(u) & \le  \tilde{O}\left[ \epsilon \log\frac{1}{\delta} + |u|T^{1/\power}(\sigma +G)\log\frac{T}{\delta}\log \frac{|u|T}{\epsilon}\right]
\end{align*}
\end{restatable}
Theorem \ref{thm:main_heavy} suggests regret with heavy-tailed gradients $g_t$ has a $\power$ dependence of $\tilde{O} (T^{1/\power})$, which is optimal \citep{bubeck2013bandits, vural2022mirror}.

\section{Dimension-free Extension} \label{sec:dim_free}
    So far, we have only considered 1-dimensional problems. In this section, we demonstrate the extension to dimension-free, which is achieved by using a reduction from \cite{cutkosky2018black}. The original reduction extends a 1-dimensional algorithm to a dimension-free one by dissecting the problem into a ``magnitude'' and a ``direction'' learner. The direction learner is a constrained OLO algorithm $\mathcal{A}^{nd}$ which outputs a vector $\bv_t$ with $\|\bv_t\|\le 1$ in response to $\bg_1,\dots, \bg_{t-1}$, while the magnitude learner is an unconstrained OLO algorithm $\mathcal{A}^{1d}$ which outputs $x_t \in \R$ in response to $\langle \bg_1,\bv_1\rangle,\dots\langle \bg_{t-1},\bv_{t-1}\rangle$. The output of the entire algorithm is $\bw_t = x_t \bv_t$. Suppose $\cA_{1d}$ and $\cA_{nd}$ have regret guarantee of $R_T^{1d}(u)$ and $R_T^{nd}(\bu)$, respectively. Then regret of the dimension-free reduction is bounded by $R_T(\bu)\le \|\bu\|R^{{nd}}_T(\bu/\|\bu\|) + R^{{1d}}_T(\|\bu\|)$. Thus, in order to apply this reduction we need to exhibit a $\cA^{1d}$ and $\cA^{nd}$ that achieves low regret on heavy-tailed losses. For the magnitude learner $\cA^{1d}$, we use can use the 1d Algorithm~\ref{alg:heavy_tail_algo} that we just developed. The remaining question is how to develop a direction learner that can handle heavy-tailed subgradients. Fortunately, this is much easier since the direction learner is constrained to the unit ball.
    
     To build this direction learner, we again apply subgradient clipping, and feed the clipped subgradients to the standard FTRL algorithm with quadratic regularizer (i.e. ``lazy'' online gradient descent). This procedure is described in  Algorithm~\ref{alg:a_nd}. Note there is no regularization implemented  in Algorithm~\ref{alg:a_nd} although $\hat{\bg}_t$ induces bias. Since $\mathcal{A}^{nd}$ runs on the unit ball, careful tuning of $\tau$ is sufficient to control the bias - a concrete demonstration of how much more intricate the unconstrained case is! Finally, the full dimension-free reduction is displayed in Algorithm \ref{alg:heavy_tail_algo_nd} with its high probability guarantee stated in Theorem~\ref{thm:heavy_nd}. The details are presented in Appendix \ref{sec:app_highdim}.

\begin{algorithm}[H]
\caption{Unit Ball Gradient clipping with FTRL}\label{alg:a_nd}
\begin{algorithmic}[1]
\Require  time horizon $T$, gradient clipping parameter $\tau$, regularizer weight $\eta$
\State Set $\eta = 1/\tau$
\For{$t=1$ to $T$}
    \State Compute $\bv_t \in \argmin_{\bv: \|\bv\| \leq 1} \sum_{i=1}^{t-1} \langle \hat{\bg}_t, \bv \rangle + \frac{1}{2\eta} \|\bv\|^2$
    \State Output $\bv_t$, receive gradient $\bg_t$
    \State Set $\hat{\bg}_t = \frac{\bg_t}{\| \bg_t \|}\min (\tau, \|\bg_t \|)$
\EndFor
\end{algorithmic}
\end{algorithm}

\begin{algorithm}[H]
\caption{Dimension-free Gradient clipping for $(\sigma, G)$ Heavy-tailed gradients}\label{alg:heavy_tail_algo_nd}
\begin{algorithmic}[1]
\Require Subgradients $\power^{th}$ moment bound $\sigma^{\power}$, 
time horizon $T$,  Set Algorithm \ref{alg:heavy_tail_algo}, \ref{alg:a_nd} as $\mathcal{A}^{1d}$, $\mathcal{A}^{nd}$.
\State Set $\sigma_{1d}=(\sigma^{\power} + 2G^{\power})^{1/\power}$ and $\tau_{1d}=T^{1/\power} (\sigma_{1d}^{\power} + G^{\power})^{1/\power}=T^{1/\power} (\sigma^{\power} + 3G^{\power})^{1/\power}$
\State Initialize $\cA^{1d}$ with parameters $\sigma\gets \sigma_{1d}$ and $\tau\gets\tau_{1d}$ 
\State Initialize $\cA^{nd}$ with parameters $\sigma\gets \sigma$ and $\tau\gets T^{1/\power} (\sigma^{\power} + G^{\power})^{1/\power}$.
\For{$t=1$ to $T$}
    \State Receive $x_{t} \in \R$ from $\mathcal{A}^{1d}$, 
    \State Receive $\bv_t \in \R^d, \|\bv_t\| \leq 1$ from $\mathcal{A}^{nd}$
    \State Play output $\bw_t = x_t \bv_t$
    \State Suffer loss $\ell_t(\bw_t)$, receive gradients $\bg_t$
    \State Send $g_t=\langle \bg_t, \bv_t \rangle$ as the $t^{th}$ gradient to $ \mathcal{A}^{1d}$
    \State Send $\bg_t$ as the $t^{th}$ gradient to $\mathcal{A}^{nd}$
\EndFor
\end{algorithmic}
\end{algorithm}
 
\begin{restatable}{Theorem}{heavynd}
\label{thm:heavy_nd}
Suppose that for all $t$, $\{\bg_t\}$ are heavy-tailed stochastic subgradients satisfying $\E[\bg_t]\in\partial  \ell_t(\bw_t)$, $\|\E[\bg_t ]\| \leq G$ and $\E[\| \bg_t - \E[\bg_t] \|^{\power}] \leq \sigma^{\power}$ for some $\power \in (1, 2]$.
Then, with probability at least $1-\delta$, Algorithm \ref{alg:heavy_tail_algo_nd} guarantees
\begin{align*}
    R_T(\bu)=\sum_{t=1}^T \ell_t(\bw_t)-\ell_t(\bu) & \leq \tilde{O} \Bigg[ \epsilon \log \frac{1}{\delta} +\|\bu\| T^{1/\power}(\sigma + G) \log\frac{T}{\delta}\log\frac{\|\bu\|T}{\epsilon}\Bigg]
\end{align*}
\end{restatable}
\textbf{Complexity Analysis}: Algorithm~\ref{alg:heavy_tail_algo_nd} requires $O(d)$ space. It also requires $O(d)$ time for all operations except solving the fixed-point equation in Algorithm~\ref{alg:sub_exp_combine} (line 5). This can be solved via binary search to arbitrary precision $\epsilon_0$ for an overall complexity of $O(d+\log(1/\epsilon_0))$. This is essentially $O(d)$ in practice since we should expect $\log(1/\epsilon_0)\le 64$.

\section{Conclusion}\label{sec:conclusion}

We have presented a framework for building parameter-free algorithms that achieve high probability regret bounds for heavy-tailed subgradient estimates. This improves upon prior work in several ways: high probability bounds were previously unavailable even for the restricted setting of \emph{bounded} subgradient estimates, while even in-expectation bounds were previously unavailable for heavy-tailed subgradients.  Our development required two new techniques: first, we described a regularization scheme that effectively ``cancels''  potentially problematic iterate-dependent  variance terms arising in standard martingale concentration arguments. This allows for high probability bounds with bounded sub-exponential estimates, and we hope may be of use in other scenarios where the iterates appear in variance calculations. The second combines clipping with another new regularization scheme that  ``cancels'' another problematic iterate-dependent \emph{bias} term. On its own, this technique actually can be used to recover in-expectation bounds for heavy-tailed estimates. 

\textbf{Limitations: } Our algorithm has several limitations that suggest open questions: first, our two regularization schemes each introduce potentially suboptimal logarithmic factors. The first one introduces a higher logarithmic dependence on $T$, while the second introduces a higher logarithmic dependence on $\|\bu\|$ because the optimal clipping parameter $\tau$ depends on $\log(\|\bu\|)$. Beyond this, our algorithms require knowledge  of the parameters $\sigma$ and $\tau$. Adapting to an unknown value of even one of these parameters remains a challenging problem.

\bibliographystyle{abbrvnat}
\bibliography{references}

\newpage

\appendix

\section{Optimistic Online Learning for Predictable Regularizer}\label{sec:app_optmistic}
Algorithm \ref{alg:sub_exp_combine} provides output $w_t$ by solving $w_t= x_t - y_t \nabla \psi_t (w_t)$, where $x_t \in \mathbb{R}, y_t \ge 0$ are output from sub-algorithms $\cA_1$ and $\cA_2$, $\psi_t(w)$ is defined in equation (\ref{eqn:reg_with_number}). Under the constants for $\psi_t(w)$ defined in Theorem \ref{thm:main_sub_exp}, the following Lemma shows the existence of solution. 
\begin{Lemma}[Existence of Solution]
\label{lemma:existence_sol}
for $x_t \in \mathbb{R}, y_t \geq 0$,
\begin{align*}
    w_t = x_t - y_t \nabla \psi_t(w_t)
\end{align*}
where
\begin{align*}
    \nabla \psi_t (w) = \sign(w) \sum_{j=1 }^{2} k_j p_j \frac{|w|^{p_j-1}}{\left( |w|^{p_j}  + X_j\right)^{1-1/p_j}}
\end{align*}
for some $ k_j, X_j > 0, p_j > 1$ and $j = 1, 2$. Then $w_t$ lies in the interval of $\bigg(x_t - y_t \sum_{j=1}^2 k_j p_j, x_t \bigg] $ when $x_t \ge 0$, and in the interval of $\bigg[x_t,  x_t + y_t \sum_{j=1}^2 k_j p_j\bigg).$ Further,
\begin{align*}
    h(w) = {w-x_t} + y_t \nabla \psi_t(w)
\end{align*}
is monotonic in $w$.
\end{Lemma}
\begin{proof}
We suppose that $x_t\ge 0$. The case $x_t<0$ is entirely identical.

\textbf{case (a): }
consider $y_t \neq 0$,
\begin{align*}
    w_t &= x_t - y_t \sign(w_t) \sum_{j=1}^{2}k_j p_j \frac{|w_t|^{p_j-1} }{\left( |w_t|^{p_j} + X_j \right)^{(p_j-1)/p_j} }\\
    \intertext{rearrange}
    \frac{x_t - w_t}{y_t } & = \sign(w_t)\sum_{j=1}^{2} k_j p_j\frac{ |w_t|^{p_j-1} }{\left( |w_t|^{p_j} + X_j \right)^{(p_j-1)/p_j} }
\end{align*}
Let $f(w_t), g(w_t)$ to be the left and right handside of the last expression. Both functions are continuous in $w_t$ for under assumption of $x_t, y_t, k_j, p_j, X_j$ for $j=1$ and $2$. When $w_t^{\ast} = x_t$:
\begin{align*}
    f(w_t^{\ast}) - g(w_t^{\ast}) = 0 - \sum_{j=1}^{2} k_j p_j\frac{ |w_t^{\ast}|^{p_j-1} }{\left( |w_t^{\ast}|^{p_j} + X_j \right)^{(p_j-1)/p_j} } \leq 0
\end{align*}
When $w_t^{\ast} = x_t - y_t \sum_{j=1}^2 k_j p_j$:
\begin{align*}
    f(w_t^{\ast}) - g(w_t^{\ast}) = \sum_{j=1}^2 k_j p_j -  \sign\left(x_t - y_t \sum_{j=1}^2 k_j p_j \right)\sum_{j=1}^{2} k_j p_j\frac{ |w_t^{\ast}|^{p_j-1} }{\left( |w_t^{\ast}|^{p_j} + X_j \right)^{(p_j-1)/p_j} } >  0
\end{align*}
By intermediate value Theorem $f(w_t) = g(w_t)$ at $w_t$ in between $x_t $ and $x_t - y_t\sum_{j=1}^2 k_j p_j$.\\
\textbf{case (b): } when $y_t = 0$, $w_t = x_t$.

Finally, by inspection the derivative of $h(w) $ with respective to $w$ is always positive, hence is monotonic in $w$ so that we can numerically solve for $h(w_t^{\ast}) = 0$ via binary search.
\end{proof}
Algorithm \ref{alg:sub_exp_combine} requires the base algorithms $\cA_1$ and $\cA_2$ to satisfy a ``second-order'' regret bound, such as provided by Algorithm 1 of \cite{cutkosky2018black}. We assume the base algorithms are designed to handle only $1-$Lipschitz losses, so the following Lemma provide a simple linear transformation that allows the base algorithm to cope with any Lipschitz constant.
\begin{Lemma}[Algorithm Transformation]
\label{lemma:regret_transform}
Suppose an algorithm $\cA$ obtains regret $\sum_{t=1}^T \langle g_t, w_t-u\rangle \le \epsilon  + R_T(u)$ for some function $R_T$ for any sequence $\{g_t\}$ such that $|g_t|\le G$. Then, given some $\bar\epsilon >0$,  consider the algorithm that plays $\bar w_t = \frac{\bar \epsilon G}{\epsilon \bar G} w_t$ in response to  subgradients $\{\bar g_t\}$ with $|\bar g_t|\le \bar G$, where $w_t$ is the output of $\cA$ on the sequence $\{g_t\}$ with $g_t = \frac{G}{\bar G} \bar g_t$. This procedure ensures regret:
\begin{align*}
    \sum_{t=1}^T  \langle \bar g_t, \bar w_t - u\rangle \le \bar \epsilon + \frac{\bar\epsilon}{\epsilon}R_T\left( \frac{\epsilon \bar G}{\bar\epsilon G} u\right)
\end{align*}
\end{Lemma}
\begin{proof}
Since $|g_t|\le G$ by construction, we have:
\begin{align*}
    \sum_{t=1}^T \langle \bar g_t, \bar w_t - u\rangle &=\frac{\bar G}{G}\sum_{t=1}^T \left\langle g_t, \frac{\bar \epsilon G}{\epsilon \bar G} w_t -u\right\rangle \\
    &=\frac{ \bar \epsilon}{\epsilon}\sum_{t=1}^T \left\langle g_t, w_t - \frac{\epsilon \bar G}{\bar\epsilon G} u\right\rangle\\
    &\le \bar \epsilon + \frac{\bar\epsilon}{\epsilon}R_T\left( \frac{\epsilon \bar G}{\bar\epsilon G} u\right)
\end{align*}
\end{proof}
Intuitively, if we instantiate this Lemma with an algorithm obtaining $R_T(u)=\epsilon + |u|\sqrt{T \log(|u|T/\epsilon)} +|u|\log(|u|T/\epsilon)$ for 1-Lipschitz losses, we can obtain for any $\epsilon$, an algorithm for $G$-Lipschitz losses with regret $\epsilon + |u|G\sqrt{T \log(|u|GT/\epsilon)} + |u|G\log(|u|GT/\epsilon)$.

We are now at the stage to prove Theorem \ref{thm:base_algo}. We
restate the Theorem for reference, followed by its proof.

\basealgo*
\begin{proof}
The proof is similar to optimistic reduction in \cite{cutkosky2019combining}, which combines regret guarantees from two online learning algorithms. First, we observe that $\cA_1$ outputs $x_t'$ from line 3 and receives gradients $(g_t + \nabla \psi_t(w_t))/(b+H) \leq 1$ from line 9 in Algorithm \ref{alg:sub_exp_combine}. Hence we apply Lemma~\ref{lemma:regret_transform} by choosing $\epsilon = \bar \epsilon$, and set $G = 1$, $\bar G = b+H$, we have the following holds for any $u$,

\begin{align*}
    R_T^1(u) & = \sum_{t=1}^{T} \langle g_t + \nabla \psi_t(w_t), x_t - u \rangle\\
    & \leq \epsilon + A |u| \sqrt{\sum_{t=1}^{T} |g_t + \nabla \psi_t(w_t) |^2 \left[1 + \log \left(\frac{(b+H)^2 |u|^2 T^C }{\epsilon^2} + 1 \right) \right]}\\
     & ~~~ + B(b+H)|u| \log \left( \frac{(b+H)|u|T^C}{\epsilon} + 1 \right)
\end{align*}
Similarly for $\cA_2$ outputs $y_t'$ and receives $ \frac{-\langle g_t + \nabla \psi_t(w_t), \nabla \psi_t(w_t) \rangle }{H(b+H)}$, Hence use Lemma \ref{lemma:regret_transform} by setting $\epsilon = \bar \epsilon$, $G=1$ and $\bar G = H(b+H)$, we have the following for all $y_\star$:
\begin{align*}
     R_T^2(y_{\star}) & = \sum_{t=1}^{T} -\langle g_t + \nabla \psi_t(w_t), \nabla \psi_t(w_t) \rangle (y_t - y_{\star})\\
    & \leq \epsilon + A |y_{\star}| \sqrt{\sum_{t=1}^{T} \langle g_t + \nabla \psi_t(w_t), \nabla \psi_t(w_t) \rangle^2 \left[1 + \log \left(\frac{(b+H)^2 H^2 |y_{\star}|^2 T^C }{\epsilon^2} + 1 \right) \right]} \\
    & ~~~ + B(b+H)H|y_{\star}| \log \left( \frac{(b+H)H |y_{\star}| T^C}{\epsilon} + 1 \right)
\end{align*}
The relationship between the $R_T^{\cA}(u)$ bounded by linearized loss and $  R_T^1(u),  R_T^2(y_{\star})$ is revealed:
\begin{align*}
    R_T^{\cA}(u) & \leq \sum_{t=1}^{T} \langle g_t + \nabla \psi_t(w_t), w_t - u \rangle\\
    &= \sum_{t=1}^{T} \langle g_t + \nabla \psi_t(w_t), x_t - u \rangle - y_t \langle g_t + \nabla \psi_t(w_t), \nabla \psi_t(w_t) \rangle\\
    & \leq \inf_{y_{\star} \geq 0} R_T^1(u) + R_T^2(y_\star) - y_{\star} \sum_{t=1}^{T}\langle g_t + \nabla \psi_t(w_t), \nabla \psi_t(w_t) \rangle
    \intertext{use identity $-2\langle a, b \rangle = \| a- b \|^2 - \|a \|^2 - \|b\|^2 $}
    & =  \inf_{y_{\star} \geq 0} R_T^1(u) + R_T^2(y_\star) + \frac{y_{\star}}{2} \sum_{t=1}^{T} | g_t |^2 - |g_t + \nabla \psi_t (w_t) |^2 - | \nabla \psi_t(w_t) |^2\\
    & \leq \inf_{y_{\star} \geq 0} 2\epsilon + A |u| \sqrt{\sum_{t=1}^{T} |g_t + \nabla \psi_t(w_t) |^2 \left[1 + \log \left(\frac{(b+H)^2 |u|^2 T^C }{\epsilon^2} + 1 \right) \right]}\\
    & ~~~ + B(b+H)|u| \log \left[ \frac{(b+H)|u|T^C}{\epsilon} + 1 \right]  + B(b+H)H|y_{\star}| \log \left[ \frac{(b+H)H |y_{\star}| T^C}{\epsilon} + 1 \right]\\
    & ~~~ + A |y_{\star}| \sqrt{\sum_{t=1}^{T} \langle g_t + \nabla \psi_t(w_t), \nabla \psi_t(w_t) \rangle^2 \left[1 + \log \left(\frac{(b+H)^2 H^2 |y_{\star}|^2 T^C }{\epsilon^2} + 1 \right) \right]}\\
    & ~~~ + \frac{y_{\star}}{2} \sum_{t=1}^{T} | g_t |^2 - |g_t + \nabla \psi_t (w_t) |^2 - | \nabla \psi_t(w_t) |^2\\
    \intertext{let $X = \sum_{t=1}^{T} |g_t + \nabla \psi_t(w_t) |^2 $}
    & \leq \inf_{y_{\star} \geq 0} \sup_{X \geq 0 } 2\epsilon + A |u| \sqrt{ X \left[1 + \log \left(\frac{(b+H)^2 |u|^2 T^C }{\epsilon^2} + 1 \right) \right]}\\
    & ~~~ + B(b+H)|u| \log \left[ \frac{(b+H)|u|T^C}{\epsilon} + 1 \right] +  B(b+H)H|y_{\star}| \log \left[ \frac{(b+H)H |y_{\star}|T^C}{\epsilon} + 1 \right]\\
    & ~~~ + A |y_{\star}| \sqrt{ XH^2 \left[1 + \log \left(\frac{(b+H)^2 H^2 |y_{\star}|^2 T^C }{\epsilon^2} + 1 \right) \right]}\\
    & ~~~ + \frac{y_{\star}}{2} \sum_{t=1}^{T} \left( | g_t |^2  - | \nabla \psi_t(w_t) |^2 \right) - \frac{y_{\star}}{2} X \\
    & \leq \inf_{y_{\star} \geq 0} \sup_{X \geq 0 } \sup_{Z \geq 0} 2\epsilon + A |u| \sqrt{ X \left[1 + \log \left(\frac{(b+H)^2 |u|^2 T^C }{\epsilon^2} + 1 \right) \right]}\\
    & ~~~ + B(b+H)|u| \log \left[ \frac{(b+H)|u|T^C}{\epsilon} + 1 \right] + B (b+H)H |y_\star| \log \left[ \frac{(b+H)H |y_{\star} |T^C}{\epsilon} + 1 \right]\\
    & ~~~ + A |y_{\star}| \sqrt{ ZH^2 \left[1 + \log \left(\frac{(b+H)^2 H^2 |y_{\star}|^2 T^C }{\epsilon^2} + 1 \right) \right]}\\
    & ~~~ + \frac{y_{\star}}{2} \sum_{t=1}^{T} \left( | g_t |^2  - | \nabla \psi_t(w_t) |^2 \right) - \frac{y_{\star}}{4} (X+Z) \\ 
    \intertext{set \begin{align*}
        y_{\star} = \min \left(\frac{2A|u| \sqrt{1 + \log((b+H)^2|u|^2T^C/\epsilon^2 + 1) }}{\sqrt{\max(0, \sum_{t=1}^{T} \left( | g_t |^2  - | \nabla \psi_t(w_t) |^2 \right))}}, \frac{|u|}{H} \right)
    \end{align*}}
    & \leq \sup_{X \geq 0 } \sup_{Z \geq 0} 2\epsilon + A |u| \sqrt{ X \left[1 + \log \left(\frac{(b+H)^2 |u|^2 T^C }{\epsilon^2} + 1 \right) \right]}\\
    & ~~~ + B(b+H)|u| \log \left[ \frac{(b+H)|u|T^C}{\epsilon} + 1 \right] - \frac{y_{\star}}{4} (X+Z)\\
    & ~~~ + B(b+H) |u| \log \left[ \frac{(b+H)|u|T^C}{\epsilon} + 1 \right]\\
    & ~~~ + A |y_{\star}| \sqrt{ ZH^2 \left[1 + \log \left(\frac{(b+H)^2 H^2 |y_{\star}|^2 T^C }{\epsilon^2} + 1 \right) \right]}\\
    & ~~~ + A|u| \sqrt{1 + \log(\frac{(b+H)^2|u|^2T^C}{\epsilon^2} + 1) }\sqrt{\max(0, \sum_{t=1}^{T} \left( | g_t |^2  - | \nabla \psi_t(w_t) |^2 \right))} \\
    \intertext{ For $a, b > 0$, $\sup_x a \sqrt{x} - bx = a^2/4b$, apply the identity to both $\sup_{X>0}, \sup_{Z>0}$ }
    & \leq  2\epsilon + A^2 |u|^2 \left[1 + \log \left(\frac{(b+H)^2 |u|^2 T^C }{\epsilon^2} + 1 \right) \right]/ y_{\star}\\
    & ~~~ + A^2 |y_{\star}| H^2 \left[1 + \log \left(\frac{(b+H)^2 H^2 |y_{\star}|^2 T^C }{\epsilon^2} + 1 \right) \right]\\
    & ~~~ + 2B(b+H)|u| \log \left[ \frac{(b+H)|u|T^C}{\epsilon} + 1 \right]\\
    & ~~~ + A|u| \sqrt{1 + \log(\frac{(b+H)^2|u|^2T^C}{\epsilon^2} + 1) }\sqrt{\max(0, \sum_{t=1}^{T} \left( | g_t |^2  - | \nabla \psi_t(w_t) |^2 \right))} \\
    \intertext{substitute $y_{\star}$}
    & \leq 2\epsilon + \frac{A}{2}|u| \sqrt{\left[1 + \log \left(\frac{(b+H)^2 |u|^2 T^C }{\epsilon^2} + 1 \right) \right] }\sqrt{\max(0, \sum_{t=1}^{T} \left( | g_t |^2  - | \nabla \psi_t(w_t) |^2 \right))}\\
    & ~~~ + A^2 |u| H \left[1 + \log \left(\frac{(b+H)^2 |u|^2 T^C }{\epsilon^2} + 1 \right) \right]+ 2B(b+H)|u| \log \left[ \frac{(b+H)|u|T^C}{\epsilon} + 1 \right]\\
    & ~~~ + A|u| \sqrt{1 + \log(\frac{(b+H)^2|u|^2T^C}{\epsilon^2} + 1) }\sqrt{\max(0, \sum_{t=1}^{T} \left( | g_t |^2  - | \nabla \psi_t(w_t) |^2 \right))} \\
    & \leq 2\epsilon + \frac{3A}{2}|u|\sqrt{\left[1 + \log \left(\frac{(b+H)^2 |u|^2 T^C }{\epsilon^2} + 1 \right) \right] }\sqrt{\max(0, \sum_{t=1}^{T} \left( | g_t |^2  - | \nabla \psi_t(w_t) |^2 \right))}\\
    & ~~~+ |u|\left(A^2H  + 2 B (b+H) \right)\left[1 + \log \left(\frac{(b+H)^2 |u|^2 T^C }{\epsilon^2} + 1 \right) \right]
\end{align*}
Define a constant $N$
    \begin{align*}
        N = 1 + \log \left(\frac{(b+H)^2 |u|^2 T^C }{\epsilon^2} + 1 \right)
    \end{align*}
Then, $R_T^{\cA}(u)$ can be written as,
    \begin{align*}
        R_T^{\cA}(u) & \leq \sum_{t=1}^{T} \langle g_t + \nabla \psi_t(w_t), w_t - u \rangle\\
        & \leq 2\epsilon + |u| \left[ \frac{3A}{2} \sqrt{N \max \left(0, \sum_{t=1}^{T} | g_t |^2  - | \nabla \psi_t(w_t) |^2 \right)} +  \left(A^2H  + 2B (b+H) \right) N \right]
    \end{align*}
\end{proof}

The following Lemma shows the magnitude of $w_t$ as a function of $t$, where $\{w_t\}$ is a sequence of output from algorithm \ref{alg:sub_exp_combine}. We shall see later on that a coarse bound for $w_t$ helps to proof Theorem \ref{thm:main_sub_exp}.

\begin{Lemma}[Exponential Growing Output]\label{lemma:coin_betting_exp_grow}
Suppose $\cA$ is an arbitrary OLO algorithm that guarantees regret $R_T^{\cA}(0)=\sum_{t=1}^T \langle \bg_t, \bw_t\rangle\le \epsilon$ for all sequence $\{\bg_t\}$ with $\|\bg_t\|\le G$. Then  it must hold that $\|\bw_t\|\le \frac{\epsilon}{2G}2^t$ for all $t$.
\end{Lemma}
\begin{proof}
We will first prove by contradiction that $\|\bw_t\|\le \frac{\epsilon - \sum_{i=1}^{t-1} \langle \bg_i,\bw_i\rangle}{G}$ for all $t$ for all sequences $\{\bg_t\}$. Suppose that there is some $t$ and sequence $\bg_1,\dots,\bg_{t-1}$ such that $\|\bw_t\|> \frac{\epsilon - \sum_{i=1}^{t-1} \langle \bg_i,\bw_i\rangle}{G}$. Then, consider $\bg_t = G\frac{\bw_t}{\|\bw_t\|}$. Then we have:
\begin{align*}
    R_t(0) = \sum_{i=1}^{t-1} \langle \bg_i, \bw_i\rangle  + \langle \bg_t, \bw_t\rangle > \epsilon
\end{align*}
which is a contradiction, and so $\|\bw_t\|\le \frac{\epsilon - \sum_{i=1}^{t-1} \langle \bg_i,\bw_i\rangle}{G}$.

Now, if we define $H_t = \epsilon - \sum_{i=1}^{t} \langle \bg_i,\bw_i\rangle$, we have $\|\bw_t\|\le \frac{H_{t-1}}{G}$. Therefore:
\begin{align*}
    H_{t} &= H_{t-1} - \langle \bg_t, \bw_t\rangle\\
    &\le 2H_{t-1}\\
    &\le 2^t H_0\\
    &=\epsilon 2^t
\end{align*}
Thus, we have $\|\bw_t\|\le \frac{H_{t-1}}{G}= \frac{\epsilon}{2G}2^t$ as desired.
\end{proof}

\section{Cancellation for Gradients with Sub-exponential Noise}\label{sec:app_bernstein}

In this Section, we ultimately provide the proof for Theorem \ref{thm:main_sub_exp}. We first show a few algebraic lemma followed by the property of the regularizer, Then we show the proof for Theorem \ref{thm:main_sub_exp} by combining different lemma with the outlines listed in Section \ref{sec:sub_exp}.

\begin{Lemma} \label{lemma:regularizer_lb}
For $x\ge 0, a > 0$, $p \geq 1$
\begin{align*}
    \frac{a}{(x+ a)^{1-\frac{1}{p}}} \geq (x + a)^{\frac{1}{p}} - x^\frac{1}{p}
\end{align*}
\end{Lemma}
\begin{proof}
\begin{align*}
    x^{\frac{1}{p}}(x+a)^{1- \frac{1}{p}} & \geq x^{\frac{1}{p}} x^{1-\frac{1}{p}} = x\\
    \intertext{rearrange}
    0 & \geq x - x^{\frac{1}{p}}(x+a)^{1- \frac{1}{p}}\\
    a & \geq (x+a) - x^{\frac{1}{p}}(x+a)^{1- \frac{1}{p}}
\end{align*}
divide both side by $(x+a)^{1-\frac{1}{p}}$, we complete the proof
\end{proof}

\begin{Lemma}\label{lemma:triangle}
For $x, a \geq 0, p \geq 1$:
\begin{align*}
    (x+a)^{1/p} \leq x^{1/p} + a^{1/p}
\end{align*}
\end{Lemma}
\begin{proof}
\begin{align*}
    x+a & = (x^{1/p})^p + (a^{1/p})^p \leq (x^{1/p} + a^{1/p})^p
\end{align*}
raise to the power of $1/p$ to complete the proof
\end{proof}

\begin{Lemma}\label{lemma:max_w_t}
For $\bx \in \R^d$, if $\| \cdot \|_p$ is the $p-$norm:
\begin{align*}
    \frac{1}{d^{1/p}}\| \bx \|_{p} \leq \| \bx \|_{\infty} \leq \| \bx \|_p
\end{align*}
\end{Lemma}
\begin{proof}
Clearly, it suffices to consider $\bx=(x_1,\dots,x_d)$ with $x_i\ge 0$ for all $i$. let $i_{\ast} = \argmax_{i} x_i$, then for all $a \geq 0$ 
\begin{align*}
    \|\bx\|_{\infty} = x_i^{\ast} = ({x_i^{\ast}}^{p})^{1/p} \leq ({x_i^{\ast}}^p + a)^{1/p} 
    \intertext{setting $a = \sum_{i \neq i^{\ast}} x_i^p$, demonstrates the upper bound.}
\end{align*}
For the lower bound:
\begin{align*}
    \frac{1}{d^{1/p}} \| \bx \|_p = \frac{1}{d^{1/p}} \left( \sum_{i=1}^{d} x_i^p \right)^{1/p} \leq \frac{1}{d^{1/p}} \left( d {x_i^{\ast}}^p  \right)^{1/p} = x_i^{\ast}
\end{align*}
\end{proof}

\begin{Lemma}\label{lemma:lazy_max}
For $x > 0$,
\begin{align*}
    \log (x + \exp(1)) \leq \max (0, \log(x)) + \exp(1)
\end{align*}
\end{Lemma}
\begin{proof}
For $x \in (0, 1]$, $\log(x) \leq 0$. Thus, the inequality holds since $\log(x+\exp(1))\le \log(1+\exp(1))\le \exp(1)$.

For $x > 1$, we have $\log(x) > 0$. Let
\begin{align*}
    h(x) & = \log(x+\exp(1))\\
    f(x) & = \log(x) + \exp(1)
\end{align*}
Taking derivatives,
\begin{align*}
    h'(x) & = 1/(x+ \exp(1))\\
    f'(x) & = 1/x
\end{align*}
Thus, $f'(x) > h'(x)$ for $x > 1$. Now, since $f(1) \ge  h(1)$, we have $f(x) > h(x)$ for $x > 1$.

Combining both case we complete the proof.
\end{proof}

\begin{Lemma}[Cumulative Huber Loss]
\label{lemma:reg_ub_lb}
    Consider $r_t$ as in equation (\ref{eqn:huber_loss}) (copied below for convience):
\begin{align*}
    r_t(w; c,p,\alpha_0) =  
    \begin{cases}
    c \left(p |w| - (p-1) |w_t| \right) \frac{|w_t|^{p - 1}}{( \sum_{i=1}^{t} |w_i|^{p} + \alpha_0^p )^{1- 1/p} }, & |w| > |w_t|\\
    c |w|^{p} \frac{1}{( \sum_{i=1}^{t} |w_i|^{p} + \alpha_0^p )^{1-1/p} }, & |w| \leq |w_t|
    \end{cases}
\end{align*}
Then, with fixed parameter $c, \alpha_0 > 0$, $p \ge 1 $,
\begin{align*}
    &\sum_{t=1}^{T} r_t(w_t) \geq c \left( \left(\sum_{t=1}^{T} |w_t|^p  + \alpha_0^p \right)^{1/p} - \alpha_0 \right) \numberthis \label{eqn: reg_LB_general}\\
    &\sum_{t=1}^{T} r_t(u) \leq cp|u| T^{1/p} \left[ \left( \log \frac{T |u|^p + \alpha_0^p }{\alpha_0^p} \right)^{(p-1)/p} +  1 \right] \numberthis \label{eqn:reg_ub}
\end{align*}
\end{Lemma}
\begin{proof}
Define index set $ A_T = \{ t: |w_t|< |u|, t = 1, \cdots, T\}$, and let $n(A_T)$ be the cardinality of $A_T$. Let $S_t = \alpha_0^p + \sum_{i=1}^{t} |w_i|^p$. First, we show lower bound for $\sum_{t=1}^{T} r_t(w)$. Since $ p \geq 1$, 
\begin{align*}
    \sum_{t=1}^{T} r_t(w_t) &= c \sum_{t=1}^{T} \frac{|w_t|^{p}}{( \sum_{i=1}^{t} |w_i|^{p} + \alpha_0^p)^{1-1/p} }\\
    & = c \sum_{t=1}^T \frac{|w_t|^p}{(|w_t|^p + S_{t-1})^{1-1/p}} \\
     \intertext{use Lemma \ref{lemma:regularizer_lb}, set $a = |w_t|^p, x = S_{t-1}, a+x = S_t$}
     & \geq c \sum_{t=1}^{T} (S_t^{1/p} - S_{t-1}^{1/p} ) \\
     & =  c (S_T^{1/p} - \alpha_0)\\
     & =  c \left((\sum_{t=1}^{T} |w_t|^p +\alpha_0^p)^{1/p} - \alpha_0 \right) 
\end{align*}
Now, we upper bound of $\sum_t r_t(u)$. We partition the sum into two terms, and bound them individually:
\begin{align*}
    \sum_{t=1}^{T} r_t(u) & \leq  c p |u| \underbrace{\sum_{\substack{t \leq T \\|w_t| \leq |u|}} \frac{|w_t|^{p - 1}}{( \sum_{i=1}^{t} |w_i|^{p} + \alpha_0^p )^{1- 1/p} }}_{A} + c |u|^{p} \underbrace{\sum_{\substack{ t \leq T \\|w_t| > |u|}} \frac{1}{( \sum_{i=1}^{t} |w_i|^{p} + \alpha_0^p )^{1- 1/p} }}_{B}
\end{align*}
First, we bound $A$:
\begin{align*}
    A & \leq \sum_{\substack{t \leq T \\|w_t| \leq |u|}} \frac{|w_t|^{p - 1}}{\left( \sum_{\substack{i\leq t, \\ |w_i| \leq |u| }} |w_i|^{p} + \alpha_0^p \right)^{1- 1/p} }
    \intertext{by Holder's inequality $\langle a, b \rangle \leq \|a \|_m \| b \|_n$, where $\frac{1}{m} + \frac{1}{n} = 1$. Set $m = p$, $n = \frac{p}{p-1}$.
    }\\
    &\leq n(A_T)^{1/p} \left(\sum_{\substack{t \leq T \\|w_t| \leq |u|}} \frac{|w_t|^{p}}{ \sum_{\substack{i\leq t, \\ |w_i| \leq |u| }} |w_i|^{p} + \alpha_0^p }\right)^{(p-1)/p}\\
    & \leq n(A_T)^{1/p} \left(\int_{\alpha_0^p}^{ \alpha_0^p + \sum_{t \in A_T } |w_t|^{p} } \frac{1}{x} dx \right)^{(p-1)/p}\\
    & = n(A_T)^{1/p} \left( \log \frac{\sum_{t \in A_T} |w_t|^p + \alpha_0^p }{\alpha_0^p} \right)^{(p-1)/p}\\
    & \leq n(A_T)^{1/p} \left( \log \frac{n(A_T) |u|^p + \alpha_0^p }{\alpha_0^p} \right)^{(p-1)/p}\\
    & \leq T^{1/p} \left( \log \frac{ |u|^p T + \alpha_0^p }{\alpha_0^p} \right)^{(p-1)/p}
\end{align*}
Now, we bound $B$:
\begin{align*}
    B & \leq \sum_{\substack{ t \leq T \\|w_t| > |u|}} \frac{1}{( \sum_{\substack{i\leq t, \\ |w_i| > |u| }} |w_i|^{p} + \alpha_0^p )^{1- 1/p} } \\
    & \leq \sum_{\substack{ t \leq T \\|w_t| > |u|}} \frac{1}{\left[ n(\{i \leq t: |w_i| > |u| \}) |u|^{p} +\alpha_0^p \right]^{1- 1/p} } \\
    & = \frac{1}{|u|^{p-1}} \sum_{\substack{ t \leq T \\|w_t| > |u|}} \frac{1}{\left[ n(\{i \leq t: |w_i| > |u| \}) + (\frac{\alpha_0}{|u|})^p \right]^{1- 1/p} }\\
    & \leq  \frac{1}{|u|^{p-1}} \int_{(\alpha_0/|u|)^p}^{ (\alpha_0/|u|)^p + (T-n(A_T)) } \frac{1}{x^{1-1/p}} dx\\
    & = \frac{p}{|u|^{p-1}} \left[(T-n(A_T))+ (\alpha_0/|u|)^p)^{1/p} - (\alpha_0/|u|)  \right]\\
    \intertext{by Lemma \ref{lemma:triangle}, set $x = T-n(A_T), a = (\alpha_0/|u|)^p$}
    & \leq \frac{p}{|u|^{p-1}} (T-n(A_T))^{1/p}\\
    & \leq \frac{p}{|u|^{p-1}} T^{1/p}
\end{align*}
Combining $A$ and $B$:
\begin{align*}
    \sum_{t=1}^{T} r_t(u) & \leq cp|u| T^{1/p} \left[ \left( \log \frac{ |u|^p T + \alpha_0^p }{\alpha_0^p} \right)^{(p-1)/p} + 1 \right] 
\end{align*}
\end{proof}

\begin{Lemma}\label{lemma:reg_ub_lb_with_number}
Consider the regularization function $\psi_t(w)$ defined in equation (\ref{eqn:reg_with_number}) with parameter $c_1,p_1,\alpha_1,c_2,p_2,\alpha_2$, displayed below for reference:
\begin{align*}
    \psi_t(w) =  r_t(w; c_1,p_1,\alpha_1)+ r_t(w; c_2,p_2,\alpha_2)
\end{align*}
When we set
\begin{align*}
        \begin{matrix}
            c_1 = 2 \sigma \sqrt{ \log\left(\frac{32}{\delta}\left[\log\left( 2^{T+1} \right)+2\right]^2\right)},  & c_2 =  32 b \log\left(\frac{224}{\delta}\left[\log\left( 1 + \frac{b}{\sigma} 2^{T+2}\right)+2\right]^2\right),\\
         p_1 = 2, ~~~~~~~~ p_2 = \log T, &\alpha_{1} = \epsilon/c_1, ~~~~~~~ \alpha_{2} = \epsilon \sigma/(4b(b+H)) \\
        \end{matrix}
    \end{align*}
as defined Theorem \ref{thm:main_sub_exp}, Then
\begin{align*}
    \sum_{t=1}^{T} \psi_t(w_t) &\geq c_1  \sqrt{\sum_{t=1}^T |w_t|^2 + \alpha_1^2}  +c_2  \max\left(\alpha_2,  \max_{t \in \{1, \cdots, T \}}|w_t| \right) -\epsilon \left( 1 + \frac{c_2  \sigma}{4b(b+H)} \right) \\
    \sum_{t=1}^{T} \psi_t(u) & \leq 2 c_1 |u| \sqrt{T} \left[ \sqrt{ \log \left( \frac{ T |u|^2 c_1^2 }{\epsilon^2} + 1  \right) } +  1 \right] + 3c_2 p_2 |u|  \left( \max \left( 0,  p_2 \left( \log \frac{ |u|}{\alpha_2 } + 1 \right) \right) + 4 \right)
\end{align*}
\end{Lemma}
\begin{proof}
The proof builds on Lemma \ref{lemma:reg_ub_lb}. We show the algebra for $r_t$ with fixed tuple of parameters $(c_i, p_i, \alpha_i)$ for $i=1,2$, respectively. The main difference is due to the value of $p_i$.

For $i = 1$:
\begin{align*}
     \sum_{t=1}^{T} r_t(w_t; c_1,p_1,\alpha_1) & \geq c_1 \left( \sqrt{\sum_{t=1}^T |w_t|^2 + \alpha_1^2} - \alpha_1 \right)\\
     & = c_1  \sqrt{\sum_{t=1}^T |w_t|^2 + \alpha_1^2} - \epsilon
\end{align*}
By equation (\ref{eqn:reg_ub})
\begin{align*}
    \sum_{t=1}^{T}  r_t(u; c_1,p_1,\alpha_1) & \leq 2 c_1 |u| \sqrt{T} \left[ \sqrt{ \log \frac{T |u|^2 + \alpha_1^2 }{\alpha_1^2} } +  1 \right] \\
    & = 2 c_1 |u| \sqrt{T} \left[ \sqrt{ \log \left( \frac{ T |u|^2 c_1^2 }{\epsilon^2} + 1  \right) } +  1 \right] \\
\end{align*}
For $i = 2$: by equation (\ref{eqn: reg_LB_general}) and Lemma \ref{lemma:max_w_t}:
\begin{align*}
     \sum_{t=1}^{T} r_t(w_t; c_2,p_2,\alpha_2) & \geq c_2  \left(  \max\left(\alpha_2,  \max_{t \in \{1, \cdots, T \}}|w_t| \right) -\alpha_2 \right)\\
     & =  c_2  \max\left(\alpha_2,  \max_{t \in \{1, \cdots, T \}}|w_t| \right) -  \frac{ \epsilon c_2  \sigma}{4b(b+H)} 
\end{align*}
By equation (\ref{eqn:reg_ub})
\begin{align*}
   \sum_{t=1}^{T} r_t(w_t; c_2,p_2,\alpha_2) & \leq  c_2 p_2 |u| \exp(1) \left[ \left( \log \frac{T |u|^{p_2} + \alpha_2^{p_2} }{\alpha_2^{p_2}} \right)^{(p_2-1)/p_2} +  1 \right]\\
   & \leq  c_2 p_2 |u| \exp(1) \left[ \left( \log \frac{T |u|^{p_2} + \exp(1) \alpha_2^{p_2} }{\alpha_2^{p_2}} \right)^{(p_2-1)/p_2} +  1 \right]\\
   & \leq  c_2 p_2 |u| \exp(1) \left[ \left( \log \frac{T |u|^{p_2} + \exp(1) \alpha_2^{p_2} }{\alpha_2^{p_2}} \right) +  1 \right]\\
   & = c_2 p_2 |u| \exp(1) \left[ \log \left( T \frac{|u|^{p_2}}{\alpha_2^{p_2}} + \exp(1) \right) +  1 \right]\\
   \intertext{Since $T \frac{|u|^{p_2} }{\alpha_2^{p_2}} > 0$, invoke Lemma \ref{lemma:lazy_max} by substituting $x = T \frac{|u|^{p_2} }{\alpha_2^{p_2}}$}
   & \leq c_2 p_2 |u| \exp(1) \left( \max \left( 0,  \log \left( \frac{ T |u|^{p_2}}{{\alpha_2}_{p_2}} \right) \right) + \exp(1) + 1 \right)\\
   & = c_2 p_2 |u| \exp(1) \left( \max \left( 0,  p_2 \log \left( \frac{\exp(1) |u|}{\alpha_2 } \right) \right) + \exp(1) + 1 \right)\\
   & = c_2 p_2 |u| \exp(1) \left( \max \left( 0,  p_2 \left( \log \frac{ |u|}{\alpha_2 } + 1 \right) \right) + \exp(1) + 1 \right)\\
   & \leq 3c_2 p_2 |u|  \left( \max \left( 0,  p_2 \left( \log \frac{ |u|}{\alpha_2 } + 1 \right) \right) + 4 \right)
\end{align*}
Combining both cases of $i = 1, 2$, we complete the proof.
\end{proof}
Now we are at the stage to prove Theorem \ref{thm:main_sub_exp}. We
restate the Theorem for reference, followed by the proof.

\mainsubexp*
\begin{proof}
The proof is a composition of concentration bounds and our Lemmas for the regularizers, following the outline in Section \ref{sec:sub_exp}. Previously, we defined $\epsilon_t = \nabla \ell_t(\bw_t) - g_t$. $|\epsilon_t| \leq 2b$ and $\E[\epsilon_t^2 ] \leq \sigma^2$.

\textbf{Step 1 :} We first derive a concentration bound for the \textsc{Noise} term defined in equation (\ref{eqn:motivation1}). Notice that $\{|u|\epsilon_i \}$ is a martingale difference sequence. Then by Lemma \ref{lemma:scaled_sum_berinstein}, with probability at least $1-\frac{\delta}{4}$,
\begin{align*}
     \left|\sum_{t=1}^{T} u \epsilon_t \right| \leq  4|u|b \log\frac{8}{\delta} + |u| \sigma \sqrt{ 2 T \log \frac{8}{\delta}} \numberthis \label{eqn:concen_1}
\end{align*}
Now, we coarsely bound the output $w_t$ from Algorithm \ref{alg:sub_exp_combine}. At each round $t$, $w_t$ is updated by solving
\begin{align*}
    w_t &= x_t - y_t \nabla \psi_t(w_t)\\
    & = \frac{x_t'}{b+H} - \frac{y_t'}{H(b+H)} \nabla \psi_t(w_t)
\end{align*}
where $x_t', y_t' $ are outputs from some algorithm in which the regret at the origin is bounded by some positive $\epsilon$ with Lipschitz constant $1$. By Lemma \ref{lemma:coin_betting_exp_grow}, 
\begin{align*}
    & |x_t| \leq \frac{\epsilon}{2(b+H)} 2^t & |y_t| \leq \frac{\epsilon}{2H(b+H)} 2^t
\end{align*}
Now, define $k=b+H$. Then, by triangle inequality and $|\nabla \psi_t(w_t)| \leq H$
\begin{align*}
    |w_t| \leq |x_t| + |y_t| |\nabla \psi_t(w_t)| \leq \frac{\epsilon}{b+H} 2^t = \frac{\epsilon}{k} 2^t \numberthis \label{eqn:w_t_exp_grow}
\end{align*}
Finally, $\{w_t \epsilon_t\}$ is a martingale difference sequence that satisfies:
\begin{align*}
    &\E[w_t^2 \epsilon_t^2 \mid g_1, \cdots, g_t] \leq |w_t|^2 \sigma^2\\
    & |w_t \epsilon_t| \leq 2|w_t|b 
\end{align*}
 where $w_t$ depends on $g_1, \cdots, g_{t-1}$ only. Hence by Proposition \ref{prop:boundtosubexp} $\{w_t \epsilon_t \}$ is $(|w_t|\sigma, 4|w_t|b) $sub-exponential. Then we apply Theorem \ref{thm:mdsconcentration} by setting $\sigmaunit = \epsilon \sigma /k$
 to obtain that with probability at least $1-\frac{\delta}{4}$
\begin{align*}
    \left|\sum_{t=1}^{T} w_t \epsilon_t\right| &\leq 2\sqrt{ \sigma^2 \sum_{t=1}^{T} w_t^2  \log\left(\frac{32}{\delta}\left[\log\left( \left[ \frac{ k }{\epsilon} \sqrt{\sum_{t=1}^{T} w_t^2 } \right]_1\right)+2\right]^2\right)}\\
    &\qquad +8\max(\epsilon \sigma /k,4 b \max_{t\le T} |w_t|) \log\left(\frac{224}{\delta}\left[\log\left(\frac{\max(\epsilon \sigma /k, 4 b\max_{t\le T} |w_t|)}{\epsilon \sigma / k}\right)+2\right]^2\right) \numberthis \label{eqn: concen_2}
\end{align*}
We now simplify the $\log \log$ term with a worst case upper bound of $|w_t|$. From equation (\ref{eqn:w_t_exp_grow}), we have
\begin{align*}
    \begin{matrix}
    \max{i \leq t} |w_i| \leq \frac{\epsilon}{k}  2^t, & {} &
    & \sum_{i=1}^{t} w_i^2 \leq \frac{\epsilon^2}{k^2} \sum_{i=1}^{t} 4^i = \frac{\epsilon^2}{k^2} \frac{4(4^t-1)}{3}
    \end{matrix}
\end{align*}
Hence
\begin{align*}
    \log \left( \left[\frac{k }{\epsilon} \sqrt{ \sum_{i=1}^{t} w_i^2} \right]_1\right) & \leq \log \left( \left[\sqrt{2 \cdot 4^T }  \right]_1 \right)
    \leq \log\left( 2^{T+1}   \right)
\end{align*}
\begin{align*}
 \log\left(\frac{ \max(\epsilon \sigma/k, 4b\max_{t\le T} |w_t|)}{\epsilon \sigma /k}\right) & \leq  \log\left( 1 + \frac{4bk\max_{t\le T} |w_t|)}{\epsilon \sigma }\right)\\
    & \leq \log\left( 1 + \frac{b}{\sigma} 2^{T+2}\right)
\end{align*}
Notice that the  double-logarithm in (\ref{eqn: concen_2}) is critical to ameliorate this exponential bound on $|w_t|$!

Substitute the above inequalities into equation (\ref{eqn: concen_2}), and combining  with equation (\ref{eqn:concen_1}) by union bound, with probability at least $1- \frac{\delta}{2}$:
\begin{align*}
    \textsc{Noise} & \leq  4 |u|b \log\frac{8}{\delta} + |u|  \sigma \sqrt{2 T \log \frac{8}{\delta}}  + 2\sqrt{ \sigma^2 \sum_{t=1}^{T} w_t^2  \log\left(\frac{32}{\delta}\left[\log\left( 2^{T+1} \right)+2\right]^2\right)}\\
    &\qquad +8\max(\epsilon \sigma /k,4 b \max_{t\le T} |w_t|) \log\left(\frac{224}{\delta}\left[\log\left( 1 + \frac{b}{\sigma} 2^{T+2}\right)+2\right]^2\right)\\
    & =  4 |u|b \log\frac{8}{\delta} + |u|  \sigma \sqrt{2 T \log \frac{8}{\delta}}  + 2\sqrt{ \sigma^2 \sum_{t=1}^{T} w_t^2  \log\left(\frac{32}{\delta}\left[\log\left(  2^{T+1} \right)+2\right]^2\right)}\\
    &\qquad + 32b \max(\epsilon \sigma/4kb, \max_{t\le T} |w_t|) \log\left(\frac{224}{\delta}\left[\log\left( 1 + \frac{b}{\sigma} 2^{T+2}\right)+2\right]^2\right) \\
    & = 4 |u|b \log\frac{8}{\delta} + |u| \sigma \sqrt{ 2 T \log \frac{8}{\delta}}  + c_1 \sqrt{\sum_{t=1}^{T} w_t^2}  + c_2 \max( \epsilon \sigma / 4kb, \max_{t \leq T}|w_t| )
    \numberthis \label{eqn:b_delta}
\end{align*}

\textbf{Step 2 :} Next, we derive a bound on $\sum_{t=1}^T \langle g_t, w_t-u\rangle$. Our approach builds upon the motivation sketched in equation (\ref{eqn:motivation2}). We define $R^{\cA}_T(u) = \sum_{t=1}^T \hat \ell_t(w_t)-\hat \ell_t(u)$. Notice that $R^{\cA}_T(u)$ can then be bounded by Theorem \ref{thm:base_algo}. Thus, we copy over equation (\ref{eqn:motivation2}) below, and apply Theorem~\ref{thm:base_algo} and Lemma \ref{lemma:reg_ub_lb_with_number} to bound the regret
\begin{align*}
    \sum_{t=1}^{T} \langle g_t, w_t-u \rangle & \leq R_T^{\mathcal{A}}(u) - \sum_{t=1}^{T} \psi_t(w_t) + \sum_{t=1}^{T} \psi_t(u) \\
    & \leq 2\epsilon +  |u| \left[ \frac{3A}{2} \sqrt{N \max \left(0, \sum_{t=1}^{T} | g_t |^2  - | \nabla \psi_t(w_t) |^2 \right)} + \left(A^2H  + 2B (b+H) \right) N \right]\\
    & ~~~~~ -c_1  \sqrt{\sum_{t=1}^T |w_t|^2 + \alpha_1^2}  - c_2  \max\left(\alpha_2,  \max_{t \in \{1, \cdots, T \}}|w_t| \right)  + \epsilon \left( 1 + \frac{c_2  \sigma}{2b(b+H)} \right) \\ 
    & ~~~~~ + 2 c_1 |u| \sqrt{T} \left[ \sqrt{ \log \frac{T |u|^2 + \alpha_1^2 }{\alpha_1^2} } +  1 \right] + 3 c_2 p_2 |u| \left( \max \left( 0,  p_2 \left( \log \frac{ |u|}{\alpha_2 } + 1 \right) \right) + 4 \right) \numberthis \label{eqn:g_t_loss}
\end{align*}
where $N = 1 + \log \left(\frac{(b+H)^2 |u|^2 T^C }{\epsilon^2} + 1 \right)$ and $A,B, C$ are some positive constants. 

\textbf{Step 3 :} As shown in equation (\ref{eqn:motivation3}), the regret is derived by combining equation (\ref{eqn:b_delta}) and (\ref{eqn:g_t_loss}). We observe that the martingale concentration from Step 1 will be cancelled by the negative regularization terms from Step 2 to  complete the proof:
\begin{align*}
        R_T(u) & \leq \epsilon \left( 3 + \frac{8\sigma}{b+H}  \log\left(\frac{224}{\delta}\left[\log\left( 1 + \frac{b}{\sigma} 2^{T+2}\right)+2\right]^2\right)  
        \right) \\
        & \quad + |u| \left[4c_1(A^2+B)N  + \frac{3A}{2} \sqrt{N \max \left(0, \sum_{t=1}^{T} | g_t |^2  - | \nabla \psi_t(w_t) |^2 \right)} \right]\\
        & \quad + |u|b \Bigg[ 2BN + 4 \log \frac{8}{\delta} +  \frac{c_2 \log T}{b}\left((A^2+2B)N + 3 \left( \max \left( 0,  \log T \left( \log \frac{ |u|}{\alpha_2 } + 1 \right) \right) + 4 \right) \right)\Bigg]\\
        & \quad + |u|\sqrt{T} \left[ 2c_1 \left(\sqrt{\log \left(\frac{T|u|c_1^2}{\epsilon^2} + 1 \right)} +1\right) + \sigma\sqrt{2 \log \frac{8}{\delta}}  \right]\numberthis\label{eqn:initialbighighprobbound}
    \end{align*}
The above holds for probability at least $ 1- \frac{\delta}{2}$.
    
\textbf{Step 4:} For the final statement, we must remove the random quantity $\sum_{t=1}^T g^2_t$ appearing in the bound. Fortunately, this is achievable via a relatively straightforward application of Bernstein-style bounds. In particular, by Lemma \ref{lemma:truncated_grad_concentration}, with probability at least $1-\delta/2$
\begin{align*}
    \sum_{t=1}^{T} |g_t|^2 &\leq \frac{3}{2}T\sigma^2 +\frac{5}{3}b^2\log\frac{2}{\delta}
\end{align*}
Thus, further upper bound equation \ref{eqn:initialbighighprobbound} by union bound, we have with probability at least $1-\delta$,
\begin{align*}
        R_T(u) & \leq \epsilon \left( 3 + \frac{8\sigma}{b+H}  \log\left(\frac{224}{\delta}\left[\log\left( 1 + \frac{b}{\sigma} 2^{T+2}\right)+2\right]^2\right)  
        \right) \\
        & \quad + |u| \left[4c_1(A^2+B)N  + \frac{3A}{2} \sqrt{N \left(\frac{3}{2}T\sigma^2 +\frac{5}{3}b^2\log\frac{2}{\delta}\right)} \right]\\
        & \quad + |u|b \Bigg[ 2BN + 4 \log \frac{8}{\delta} +  \frac{c_2 \log T}{b}\left((A^2+2B)N + 3 \left( \max \left( 0,  \log T \left( \log \frac{ |u|}{\alpha_2 } + 1 \right) \right) + 4 \right) \right)\Bigg]\\
        & \quad + |u|\sqrt{T} \left[ 2c_1 \left(\sqrt{\log \left(\frac{T|u|c_1^2}{\epsilon^2} + 1 \right)} +1\right) + \sigma\sqrt{2 \log \frac{8}{\delta}}  \right]\numberthis\label{eqn:finalbighighprobbound}
    \end{align*}
\end{proof}
\section{Gradient Clipping for Heavy-tailed Gradients}\label{sec:app_clipping}
First, we show the property of truncated heavy-tailed gradients followed by the proof of Theorem \ref{thm:main_heavy}. These elementary facts can be found in \cite{zhang2020adaptive}, but we reproduce the proofs for completeness.
\begin{Lemma}[Clipped Gradient Properties]\label{lemma:clip_grad_property}
Suppose $\bg_t$ is heavy-tailed random vector, $\|\E[\bg_t]\|\le G$, $\E[\| \bg_t - \E[\bg_t]\|^\power]\le \sigma^\power$ for some $\power \in (1,2]$ and $\sigma \leq \infty$. Define truncated gradient $\hat{\bg}_t$ with a positive clipping parameter $\tau$:
\begin{align*}
    \hat{\bg}_t = \frac{\bg_t}{\|\bg_t\|} \min(\tau, \|\bg_t\|)
\end{align*}
Let $\boldsymbol{\mu} = \E[\bg_t]$. Then:
\begin{align*}
    \|\E[\hat{\bg}_t] - \boldsymbol{\mu}\| & \le \frac{2^{\power -1 } (\sigma^{\power} + G^{\power})}{\tau^{\power - 1}} \\
    \E[\|\hat{\bg}_t \|^2] & \le 2^{\power-1}{\tau^{2- \power}}(\sigma^{\power} +G^{\power})
\end{align*}
\end{Lemma}

\begin{proof}
By Jensen's inequality
\begin{align*} 
    \| \E[\hat{\bg}_t] - \boldsymbol{\mu} \|  & \leq \E[\| \hat{\bg}_t - \bg_t \|]\\
    & \leq \E [\| \bg_t \| \one [\| \bg_t \| \geq \tau]]\\
    & \leq \E[\| \bg_t\|^{\power} / \tau^{\power-1}]\\
    & \leq \E[ ( \| \bg_t - \boldsymbol{\mu} \| + \| \boldsymbol{\mu} \|)^{\power} / \tau^{\power-1}]\\
    & =  \frac{2^{\power}}{\tau^{\power-1}} \E[(\frac{1}{2}\| \bg_t - \boldsymbol{\mu} \| + \frac{1}{2}\|  \boldsymbol{\mu} \|  )^2]\\
    & \leq \frac{2^{\power-1}}{\tau^{\power-1}} \left(\E[\| \bg_t - \boldsymbol{\mu} \|^2 ] + \E[\| \boldsymbol{\mu} \|^2]  \right)\\
    & \leq \frac{2^{\power -1 } (\sigma^{\power} + G^{\power})}{\tau^{\power - 1}}
\end{align*}
The second last inequality was due to convexity and linearity of expectation. In term of the variance, the algebra is similar:
\begin{align*}
    \E[\| \hat{\bg}_t \|^2] & \leq \E[\| \bg_t \|^\power \tau^{2-\power}]\\
    & \leq \E [(\| \bg_t - \boldsymbol{\mu} \| + \| \boldsymbol{\mu} \| )^{\power} \tau^{2-\power} ]\\
    & = 2^{\power} \tau^{2-\power} \E [( \frac{1}{2} \| \bg_t - \boldsymbol{\mu} \| + \frac{1}{2} \| \boldsymbol{\mu} \| )^{\power} ]\\
    & \leq 2 \left( \E [\| \bg_t - \boldsymbol{\mu} \|^{\power}] + \E[\| \boldsymbol{\mu} \|^{\power} ]\right)\\
    & \leq 2^{\power-1}{\tau^{2- \power}}(\sigma^{\power} +G^{\power})
\end{align*}
\end{proof}

We now restate Theorem \ref{thm:main_heavy} followed by its proof.
\mainheavy*
\begin{proof}
We copy over $\phi(w)$ and regret formula in equation (\ref{eqn:lin_loss_heavy_tail}) and (\ref{eqn:motivation_heavy_2}) from section \ref{sec:heavy-tailed} here, as the analysis will be following the cancellation-by-regularization strategy described in section \ref{sec:heavy-tailed}.
\begin{align*}
    & \phi(w) = 2^{\power-1}(\sigma^{\power} + G^{\power}) |w| /\tau^{\power-1} & 
    \hat{\ell}_t (w) = \langle \E[\hat{g}_t] , w-u \rangle + \phi(w)
\end{align*}
\begin{align*}
    R_T(u) 
    & \leq \underbrace{ \sum_{t=1}^{T} \left(\langle \nabla \ell_t(w_t) - \E [ \hat{g}_t], w_t - u \rangle - \phi(w_t) + \phi(u)\right)}_{D} + \underbrace{\sum_{t=1}^{T} \hat{\ell}_t(w_t) - \hat{\ell}_t(u)}_{E}
\end{align*}
We split the regret into two parts. The term $D$ is controlled by cancellation-by-regularization through careful choice of $\phi_t(w)$ and clipping parameter $\tau$. Term $E$ is controlled in high probability through Algorithm \ref{alg:sub_exp_combine} (which uses a \emph{different} cancellation-by-regularization strategy) by sending $\hat{g}_t + \nabla \phi(w_t)$ as the $t^{th}$ subgradient. Specifically, we can view $\hat{g}_t + \nabla \phi(w_t)$ as a sub-exponential and bounded noisy gradient and $\E[\hat{g}_t + \nabla \phi(w_t)] = \nabla \hat{\ell}_t (w_t)$, so that Theorem \ref{thm:main_sub_exp} provides a high probability bound for term $E$.

\textbf{First, we bound $D$,} and show it's independent of $|w_t|$
\begin{align*}
    D &\leq \sum_{t=1}^{T} |\nabla \ell_t(w_t) - \E[\hat{g}_t] | (|w_t| + |u|) - 2^{\power-1}(\sigma^{\power} + G^{\power})  /\tau^{\power-1} \sum_{t=1}^{T} |w_t| + 2^{\power-1}(\sigma^{\power} + G^{\power}) T |u| /\tau^{\power-1}\\
    \intertext{since $\nabla \ell_t(w_t) = \E[g_t]$, by Lemma \ref{lemma:clip_grad_property}, $|\nabla \ell_t(w_t) - \E[\hat{g}_t]| \leq 2^{\power-1}(\sigma^{\power} + G^{\power})  /\tau^{\power-1}$.}
    & \leq 2^{\power} T|u| (\sigma^{\power} + G^{\power})/\tau^{\power-1}\\
    \intertext{set $\tau = T^{1/\power} (\sigma^{\power}+G^{\power})^{1/\power}$}
    & = 2^{\power}|u| T^{1/\power} (\sigma^{\power} + G^{\power})^{1/\power}\\
    & \leq 4 |u| T^{1/\power} (\sigma^{\power} + G^{\power})^{1/\power}
\end{align*}
\textbf{Now we bound $E$} in high probability with $\tau = T^{1/\power}(\sigma^{\power}+G^{\power})^{1/\power}$. We sometimes will substitute the value of $\tau$ and sometimes leave it as it is during the derivation for convenience. Define the noise as $\epsilon_t$,
\begin{align*}
    \epsilon_t &= \nabla \hat{\ell}_t(w_t) - (\hat{g}_t + \nabla \phi(w_t)) = \E[\hat{g}_t] - \hat{g}_t
\end{align*}
From the definition of gradient clipping, $|\hat{g}_t| \leq \tau$. Also by Lemma \ref{lemma:clip_grad_property}, 
\begin{align*}
    \E[ \hat{g}_t^2 | w_t] \leq 2\tau^{2-\power} (\sigma^{\power} + G^{\power}) = 2\tau^2 T^{-1}
\end{align*}
Hence term $E$ can be bounded by Theorem \ref{thm:main_sub_exp}, where we set the following constants for Algorithm \ref{alg:sub_exp_combine},
\begin{align*}
        \begin{matrix}
            c_1 = 2 \tau \sqrt{ \frac{2}{T}\log\left(\frac{32}{\delta}\left[\log\left( 2^{T+1} \right)+2\right]^2\right)},  & c_2 =  32 \tau \log\left(\frac{224}{\delta}\left[\log\left( 1 + \sqrt{T} 2^{T+5/2}\right)+2\right]^2\right),\\
         p_1 = 2, ~~~~~~~~~~~~ p_2 = \log T, &\alpha_{1} = \epsilon/c_1, ~~~~~~~ \alpha_{2} = (\sqrt{2}\epsilon) /(4\sqrt{T}(\tau+H)) \\
        \end{matrix}
    \end{align*}
where $H = c_1p_1+c_2p_2$. Let $N = 1 + \log \left(\frac{(\tau+H)^2 |u|^2 T^C }{\epsilon^2} + 1 \right)$. Then by equation (\ref{eqn:initialbighighprobbound}), with probability at least $1-\frac{\delta}{2}$ for some positive $A,B,C$,
\begin{align*}
        E & \leq \epsilon \left( 3 + \frac{8 \tau \sqrt{2/T}}{\tau+H}  \log\left(\frac{224}{\delta}\left[\log\left( 1 + \sqrt{T} 2^{T+5/2}\right)+2\right]^2\right)  
        \right) \\
        & \quad + |u| \left[4c_1(A^2+B)N  + \frac{3A}{2} \sqrt{N \max \left(0, \sum_{t=1}^{T} | g_t |^2  - | \nabla \psi_t(w_t) |^2 \right)} \right]\\
        & \quad + |u|\tau \Bigg[ 2BN + 4 \log \frac{8}{\delta} +  \frac{c_2 \log T}{\tau}\left( (A^2+2B)N + 3 \left( \max \left( 0,  \log T \left( \log \frac{ |u|}{\alpha_2 } + 1 \right) \right) + 4 \right)  \right)\Bigg]\\
        & \quad + |u|\sqrt{T} \left[ 2c_1 \left(\sqrt{\log \left(\frac{T|u|c_1^2}{\epsilon^2} + 1 \right)} +1\right) + 2 \tau \sqrt{\frac{1}{T} \log \frac{8}{\delta}}  \right]
    \end{align*}
Combining $D$, $E$ and substitute $\tau$ when convenient and group in terms of the product of $|u|$ with $T$ to some power of $\power$
\begin{align*}
    R_T(\bu) &\leq \epsilon \left( 3 + \frac{8 \tau \sqrt{2/T}}{\tau+H}  \log\left(\frac{224}{\delta}\left[\log\left( 1 + \sqrt{T} 2^{T+5/2}\right)+2\right]^2\right)  
        \right) \\
    & \quad + |u| \frac{3A}{2} \sqrt{N \max\left(0, \sum_{t=1}^{T}  | \hat{g}_t + \nabla \phi(w_t) |^2  - | \nabla \psi_t(w_t) |^2 \right)} \\
    & \quad + |u|T^{1/\power}(\sigma^{\power} + G^{\power})^{1/\power} \Bigg[ 2BN + 2 \sqrt{\log \frac{8}{\delta}} + 4 \log \frac{8}{\delta} \\
    & \qquad + 4\sqrt{2 \log \left( \frac{32}{\tau} [\log(2^{T+1} ) +2]^2 \right)} \left( \frac{1}{\sqrt{T}}(A^2 + 2B)N +\sqrt{\log \left(\frac{T|u|c_1^2}{\epsilon^2} + 1 \right)} +1 \right)\\
    & \qquad + \frac{c_2 \log T}{\tau} \left(2 (A^2+B)N + 3\max \left( \max \left( 0,  \log T \left( \log \frac{ |u|}{\alpha_2 } + 1 \right) \right) + 4 \right)  + 4 \right) \Bigg] \numberthis \label{eqn:thm_heavy_1}
\end{align*}
For the final statement, notice that although $\hat{g}_t$ is a random quantity, we can bound $\sum_{t=1}^T |\hat g_t|^2$ with high probability. By Lemma \ref{lemma:clip_grad_property}, $\E[\hat{g}_t^2] \leq 2^{\power-1} \tau^{2-\power} (\sigma^\power + G^{\power})$ and note $|\hat{g}_t| \leq \tau$. Thus by Lemma \ref{lemma:truncated_grad_concentration}, with probability at least $1-\delta/2$
\begin{align*}
    \sum_{t=1}^{T} |\hat{g}_t|^2 &\leq 3T 2^{\power -2} \tau^{2-\power}(\sigma^{\power} + G^{\power}) + 2\tau^2 \log \frac{2}{\delta}\\
    & \leq \tau^2(3+2 \log\frac{2}{\delta})
\end{align*}
Finally, since equation (\ref{eqn:thm_heavy_1}) holds for probability as least $1-\frac{\delta}{2}$, we further upperbound $|\hat{g}_t|^2$ by union bound for our final regret guarantee with probability at least $1-\delta$
\begin{align*}
     R_T(u) &\leq \epsilon \left( 3 + \frac{8 \tau \sqrt{2/T}}{\tau+H}  \log\left(\frac{224}{\delta}\left[\log\left( 1 + \sqrt{T} 2^{T+5/2}\right)+2\right]^2\right)  
        \right) \\
    & \quad + |u| \frac{3A}{2} \sqrt{N \max\left(0,  \tau^2(3+2 \log \frac{2}{\delta}) + \sum_{t=1}^{T} |\nabla \phi(w_t) |^2  - | \nabla \psi_t(w_t) |^2 \right)} \\
    & \quad + |u|T^{1/\power}(\sigma^{\power} + G^{\power})^{1/\power} \Bigg[ 2BN + 2 \sqrt{\log \frac{8}{\delta}} +  4 \log \frac{8}{\delta} \\
    & \qquad + 4 \sqrt{ 2 \log\left(\frac{32}{\delta}\left[\log\left( 2^{T+1} \right)+2\right]^2\right)}\left( \frac{1}{\sqrt{T}}(A^2 + 2B)N +\sqrt{\log \left(\frac{T|u|c_1^2}{\epsilon^2} + 1 \right)} +1 \right)\\
    & \qquad + \frac{c_2 \log T}{\tau} \left(2 (A^2+B)N + 3\max \left( 0,  \log T \log \left( \frac{3 |u|}{\alpha_2 } \right) \right) + 4 \right) \Bigg]
\end{align*}
where $ |\nabla \phi(w_t)| \leq 2^{\power-1} T^{\frac{1}{\power} - 1} (\sigma^{\power} + G^{\power})^{\frac{1}{\power}} $
\end{proof}

\section{Dimension-free Gradient Clipping for Heavy-tailed Gradients}\label{sec:app_highdim}

\begin{Lemma}(Unit Ball Domain Algorithm High Probability)\label{lemma:prob_A_nd}
 Suppose $\{\bg_t\}$ is a sequence of  heavy-tailed stochastic gradient vectors such that $\E[|\bg_t|] \leq G$, $\E[\|\bg_t - \E[\bg_t] \|^\power] \leq \sigma^{\power}$ for some $\power \in (1, 2]$. Let $\hat{\bg}_t$ be the clipped gradient $\hat{\bg}_t = \bg_t / \| \bg_t \| \min(\tau, \|\bg_t\|)$, where $\tau$ is set as $T^{1/\power} (\sigma^{\power} + G^{\power})^{1/\power}$. The constrained domain on unit ball ensures  $\|\bw_t\|, \|\bu \| \leq 1$. Then with probability at least $1-\delta$, algorithm \ref{alg:a_nd} guarantees 
 \begin{align*}
    R_T^{nd}(\bu) & \leq \sum_{t=1}^{T} \langle \E[\bg_t], \bw_t - \bu \rangle\\
    & \leq T^{1/\power}(\sigma^{\power} + G^{\power})^{1/\power} \Bigg( \frac{1}{2}\|\bu\|^2 + \left( \frac{3}{2} + \frac{20}{3} \log\frac{2}{\delta} \right)  + 15  \sqrt{ \log \frac{160}{\delta} } \\
    & \qquad  + 184 \log \left(\frac{448}{\delta} \left[ \log \left( 2 + \frac{16}{\tau}\right) + 1 \right]^2 \right) + 4 \Bigg) 
 \end{align*}
\end{Lemma}
\begin{proof}
The analysis follows similar to the $1d$ analysis as seen in Theorem \ref{thm:main_heavy}. We run a standard Follow-the-regularized-leader (FTRL) algorithm with $L_2$ regularization on clipped gradient $\hat{\bg}_t$ instead of the true gradient $\nabla \ell_t(\bw_t)$. A `bias' term was introduced by $\E[\bg_t] - \E[\hat{\bg}_t]$ which can be regulated to be sublinear by the clipping parameter $\tau$. In addition, a `noise' term due to $\E[\hat{\bg}_t] - \hat{\bg}_t$ can be bounded with high probability. We decompose the regret and label the corresponding parts below,
 \begin{align*}
    R_T^{nd}(\bu) &= \sum_{t=1}^{T} \langle \nabla \ell_t(\bw_t), \bw_t - \bu \rangle\\
    & = \sum_{t=1}^{T} \langle \hat{\bg}_t, \bw_t - \bu \rangle + \sum_{t=1}^{T} \langle \nabla \ell_t(\bw_t)  -\E[\hat{\bg}_t], \bw_t - \bu \rangle + \sum_{t=1}^{T} \langle \E[\hat{\bg}_t] - \hat{\bg}_t,  \bw_t - \bu \rangle\\
    & \leq \underbrace{\| \sum_{t=1}^{T} \langle \hat{\bg}_t, \bw_t - \bu \rangle \|}_{\text{FTRL}} + \underbrace{\sum_{t=1}^{T} \langle \nabla \ell_t(\bw_t)  -\E[\hat{\bg}_t], \bw_t - \bu \rangle}_{\text{`bias'}} + \underbrace{\| \sum_{t=1}^{T} \langle \E[\hat{\bg}_t] - \hat{\bg}_t,  \bw_t - \bu \rangle \|}_{\text{`noise'}}
\end{align*}
Now we will bound the regret in three step.

\textbf{Step 1 :}
For the part controlled by a FTRL algorithm with fixed $L_2$ regularization weight $\eta$ (see Corollary 7.8 in \cite{orabona2019modern})
\begin{align*}
    \sum_{t=1}^{T} \langle \hat{\bg}_t, \bw_t - \bu \rangle &\leq \frac{1}{2\eta} \|\bu \|^2 + \frac{\eta}{2} \sum_{t=1}^{T} \| \hat{\bg}_t \|^2
    \intertext{Since $\|\hat{\bg}_t \| \leq \tau$, and by Lemma \ref{lemma:clip_grad_property}, $\E[\|\hat{\bg}_t \|^2 ]\leq 2^{\power -1} \tau^{2-\power} (\sigma^{\power} +G^{\power}) \leq 2\tau/T$. By Proposition \ref{prop:boundtosubexp}, $\bg_t$ is $(\tau \sqrt{2/T}, 2 \tau)$ sub-exponential. Hence by Lemma \ref{lemma:truncated_grad_concentration}, with probability at least $1-\delta$}
    & \leq \frac{1}{2\eta} \|\bu \|^2 +  \frac{\eta}{2} \left(3\tau^2 + \frac{20}{3}\tau^2 \log \frac{1}{\delta} \right)\\
    \intertext{set $\eta = 1/\tau$}
    & \leq \frac{\tau}{2}\|\bu\|^2 + \tau \left( \frac{3}{2} + \frac{20}{3} \log\frac{1}{\delta} \right)
\end{align*}
\textbf{Step 2 :}
For the `bias' term, note $\nabla \ell_t(\bw_t) = \E[\bg_t]$, by Lemma \ref{lemma:clip_grad_property}, $\|\E[\hat{\bg}_t] - \E[\bg_t]\| \le \frac{2^{\power-1} (\sigma^{\power} + G^{\power}) }{\tau^{\power-1}}$, and the constrained domain suggests $\|\bw_t - \bu \| \leq 2$
\begin{align*}
    \sum_{t=1}^{T} \langle \nabla \ell_t(\bw_t)  -\E[\hat{\bg}_t], \bw_t - \bu \rangle & \leq \frac{2^{\power}T}{\tau^{\power -1 }}(\sigma^{\power} + G^{\power})
\end{align*}
\textbf{Step 3 :}
For a high probability bound for the `noise' term, let $X_t = \langle \E[\hat{\bg_t}] - \hat{\bg}_t ,  \bw_t - \bu \rangle$, hence $\{ X_t\}$ is a vector valued MDS adapted to filtration $\mathcal{F}_t$ with the following bound almost surely,
\begin{align*}
    X_t & \leq( \E[\hat{\bg_t}]  + \|\hat{\bg}_t\| \| ) \bw_t - \bu \|\\
    & \leq 2\tau \|  \bw_t - \bu \|\\
\end{align*}
\begin{align*}
    \E[\|X_t\|^2 \mid \mathcal{F}_t] & \leq \|\bw_t - \bu \|^2 \E[\| \hat{\bg}_t \|^2 \mid \mathcal{F}_{t}] \\
    \intertext{by Lemma \ref{lemma:clip_grad_property}}
    & \leq \|\bw_t - \bu \|^2 2^{\power -1} \tau^{2-\power}(\sigma^{\power} + G^{\power})\\
    & \leq 2 \|\bw_t - \bu \|^2 T^{2/\power - 1}(\sigma^{\power} + G^{\power})^{2/\power}\\
    & = 2 \| \bw_t - \bu \|^2 \tau^2 T^{-1}
\end{align*}
and both bounds are $\mathcal{F}_{t-1}$ measurable. By proposition \ref{prop:boundtosubexp} $ \{\langle \E[\hat{\bg}_t] - \hat{\bg}_t,  \bw_t - \bu \rangle \}$ is $\{ \sqrt{ 2} \|\bw_t - \bu \| \tau T^{-1/2}
, 4 \tau \|  \bw_t - \bu \| \} $ sub-exponential noise. Use Theorem \ref{thm:vectorconcentration} and set $\sigmaunit = \tau$, with probability at least $1-\delta$,
\begin{align*}
    \left\|\sum_{t=1}^{T} X_t\right\|&\le 5\sqrt{2 \frac{\tau^2}{T} \sum_{t=1}^T \|\bw_t - \bu \|^2 \log\left(\frac{16}{\delta}\left[\log\left(\left[\sqrt{ \frac{2}{T} \sum_{t=1}^{T}\|\bw_t - \bu \|^2 }\right]_1\right) +1\right]^2\right)}\\
    &\qquad + 23 \max(\tau, 4\tau \max_{t\le T} \| \bw_t - \bu \| )\log\left(\frac{224}{\delta}\left[\log\left(2 \max(1, \frac{4}{\tau} \max_{t\le T} \| \bw_t  - \bu \|)\right)+1\right]^2\right)\\
    & \le 15 \tau \sqrt{ \log \frac{80}{\delta} } + 184 \tau \log \left(\frac{224}{\delta} \left[ \log \left( 2 + \frac{16}{\tau}\right) + 1 \right]^2 \right) \\
    & = T^{1/\power}(\sigma^{\power} + G^{\power})^{1/\power} \left( 15 \sqrt{ \log \frac{80}{\delta} } + 184 \log \left(\frac{224}{\delta} \left[ \log \left( 2 + \frac{16}{\tau}\right) + 1 \right]^2 \right) \right)
\end{align*}
\textbf{Composition :}
Combining the high probability bound from step 1 and 3 by union bound and a deterministic bound from step 2, with probability at least $1-\delta$, we have the following regret guarantee, 
\begin{align*}
    R_T^{nd}(\bu)  &\leq \frac{\tau}{2}\|\bu\|^2 + \tau \left( \frac{3}{2} + \frac{20}{3} \log\frac{2}{\delta} \right) +  \frac{2^{\power}T}{\tau^{\power -1 }}(\sigma^{\power} + G^{\power}) \\
    & \qquad +T^{1/\power}(\sigma^{\power} + G^{\power})^{1/\power} \left( 15 \sqrt{ \log \frac{160}{\delta} } + 184 \log \left(\frac{448}{\delta} \left[ \log \left( 2 + \frac{16}{\tau}\right) + 1 \right]^2 \right) \right)
    \intertext{substitute $\tau = T^{1/\power} (\sigma^{\power} + G^{\power})^{1/\power}$ and group terms by factorizing some power of $T$}
    & \leq T^{1/\power}(\sigma^{\power} + G^{\power})^{1/\power} \Bigg( \frac{1}{2}\|\bu\|^2 + \left( \frac{3}{2} + \frac{20}{3} \log\frac{2}{\delta} \right)  + 15  \sqrt{ \log \frac{160}{\delta} } \\
    & \qquad+ 184 \log \left(\frac{448}{\delta} \left[ \log \left( 2 + \frac{16}{\tau}\right) + 1 \right]^2 \right) + 4 \Bigg) 
\end{align*}
\end{proof}
We restate Theorem \ref{thm:heavy_nd} for reference followed by its proof \heavynd*
\begin{proof}
This result  relies on the reduction from dimension-free learning to 1d learning presented in Theorem 2 of \cite{cutkosky2018black}. This result implies that the regret \ref{alg:heavy_tail_algo_nd} can be be bounded as:
\begin{align*}
    \sum_{t=1}^{T} \ell_t(\bw_t) - \ell_t(\bu) & \leq  R_T^{1d}(\|\bu\|) + \|\bu \| R_T^{nd} (\bu / \| \bu \|)
\end{align*}
where 
\begin{align*}
    R_T^{1d}(\|u\|)&=\sum_{t=1}^T \langle g^{1d}_t, x_t - \|u\|\rangle\\
    R_T^{nd}(\bu/\|\bu\|)&=\sum_{t=1}^T \langle \bg_t, \bv_t - \bu/\|\bu\|\rangle
\end{align*}

First, by setting Algorithm \ref{alg:heavy_tail_algo} as $\mathcal{A}^{1d}$ is , Theorem \ref{thm:main_heavy} provides a bound for $R_T^{1d} (\| \bu\|)$ with appropriate parameters set in the theorem. We run $\mathcal{A}^{1d}$ on subgradients $g_t = \langle \bg_t, \bv_t \rangle$, and $\| \bv_t \| \leq 1$. We show $|\E[g_t]|$ and $\E[\| g_t - \E[g_t] \|^{\power}]$ are bounded. Notice that $g_t$ only depends on $\bg_1,\cdots, \bg_{t-1}$. Thus, by tower rule,  $\E[g_t] = \E[ \langle \bg_t, \bv_t \rangle] = \E[  \langle \E[\bg_t], \bv_t \rangle \mid \bg_1, \cdots, \bg_{t-1}]$. Then, since $\|\bv_t\| \leq 1$ we have $| \langle \E[\bg_t], \bv_t \rangle|  \leq G $. Now we are left to show a bounded central moment:
\begin{align*}
    \E[|g_t -\E[g_t]|^{\power}] & = \E[| \langle \bg_t - \E[\bg_t], \bv_t \rangle + \langle \E[\bg_t], \bv_t \rangle - \E[\langle \bg_t , \bv_t \rangle] |^{\power}]\\
    & \leq \E[| \langle \bg_t - \E[\bg_t], \bv_t \rangle|^{\power} ]+ |\langle \E[\bg_t], \bv_t \rangle|^{\power} + | \E[\langle \bg_t , \bv_t \rangle]|^{\power}\\
    & \leq \E[\| \bg_t -\E[\bg_t]\|^{\power} ]+ 2 \| \E[\bg_t] \|^{\power}\\
    & \leq \hat{\sigma}^{\power} + 2G^{\power} = \sigma^{\power}
\end{align*}
Then by Theorem \ref{thm:main_heavy}, when Algorithm \ref{alg:heavy_tail_algo} is run on $g_t = \langle \bg_t, \bv_t \rangle$,
we obtain  with probability at least $1-\delta$
\begin{align*}
      R_T^{1d}(\|\bu\|) &\leq \epsilon \left( 3 + \frac{8 \tau \sqrt{2/T}}{\tau+H}  \log\left(\frac{224}{\delta}\left[\log\left( 1 + \sqrt{T} 2^{T+5/2}\right)+2\right]^2\right)  
        \right) \\
    & \quad + \|\bu\| \frac{3A}{2} \sqrt{N \max\left(0,  \tau^2(3+2 \log\frac{2}{\delta}) + \sum_{t=1}^{T} |\nabla \phi(w_t) |^2  - | \nabla \psi_t(w_t) |^2 \right)} \\
    & \quad + \|\bu\|T^{1/\power}(\sigma^{\power} + G^{\power})^{1/\power} \Bigg[ 2BN + 2 \sqrt{\log \frac{8}{\delta}} + 4 \log \frac{8}{\delta} \\
    & \qquad + 4 \sqrt{ 2 \log\left(\frac{32}{\delta}\left[\log\left( 2^{T+1} \right)+2\right]^2\right)}\left( \frac{1}{\sqrt{T}}(A^2 + 2B)N +\sqrt{\log \left(\frac{T\|\bu\|c_1^2}{\epsilon^2} + 1 \right)} +1 \right)\\
    & \qquad + \frac{c_2 \log T}{\tau} \left(2 (A^2+B)N + 3\max \left( 0,  \log T \log \left( \frac{3 \|\bu\|}{\alpha_2 } \right) \right) + 4 \right) \Bigg] \numberthis \label{eqn:thm_heavy_nd}
\end{align*}
In terms of $R_T^{nd}(\bu/ \|\bu \|)$, Lemma \ref{lemma:prob_A_nd} implies that with probability at least $1-\delta/2$,
\begin{align*}
    R_T^{nd}(\bu/ \|\bu \|) & \leq T^{1/\power}(\hat{\sigma}^{\power} + G^{\power})^{1/\power} \Bigg( \frac{9}{2}  + \left( \frac{3}{2} + \frac{20}{3} \log\frac{4}{\delta} \right)  + 15  \sqrt{ \log \frac{320}{\delta} } \\
    & \qquad+ 184 \log \left(\frac{896}{\delta} \left[ \log \left( 2 + \frac{16}{\tau}\right) + 1 \right]^2 \right) \Bigg) \numberthis \label{eqn:thm_heavy_nd_2}
\end{align*}
Finally, replace $\delta$ as $\delta/2$ in equation (\ref{eqn:thm_heavy_nd}), then combining with equation (\ref{eqn:thm_heavy_nd_2}), we have the regret guarantee with probability at least $1-\delta$
\begin{align*}
    R_T(\bu) & \leq R_T^{1d}( \|\bu \|) + \|\bu\| R_T^{nd}(\bu/ \|\bu \|)\\
    &\leq  \epsilon \left( 3 + \frac{8 \tau \sqrt{2/T}}{\tau+H}  \log\left(\frac{448}{\delta}\left[\log\left( 1 + \sqrt{T} 2^{T+5/2}\right)+2\right]^2\right)  
        \right) \\
    & \quad + \|\bu\| \frac{3A}{2} \sqrt{N \max\left(0,  \tau^2(3+2 \log\frac{4}{\delta}) + \sum_{t=1}^{T} |\nabla \phi(w_t) |^2  - | \nabla \psi_t(w_t) |^2 \right)} \\
    & \quad + \|\bu\|T^{1/\power}(\sigma^{\power} + G^{\power})^{1/\power} \Bigg[ 2BN + 2 \sqrt{\log \frac{16}{\delta}} + 4 \log \frac{16}{\delta} \\
    & \qquad + 4 \sqrt{ 2 \log\left(\frac{64}{\delta}\left[\log\left( 2^{T+1} \right)+2\right]^2\right)}\left( \frac{1}{\sqrt{T}}(A^2 + 2B)N +\sqrt{\log \left(\frac{T\|\bu\|c_1^2}{\epsilon^2} + 1 \right)} +1 \right)\\
    & \qquad + \frac{c_2\log T}{\tau}  \left(2 (A^2+B)N + 3\max \left( 0,  \log T \log \left( \frac{3 \|\bu\|}{\alpha_2 } \right) \right) + 4 \right) \Bigg]\\
    & \quad + \|\bu \| T^{1/\power}(\hat{\sigma}^{\power} + G^{\power})^{1/\power} \Bigg[ \frac{9}{2}  + \left( \frac{3}{2} + \frac{20}{3} \log\frac{4}{\delta} \right)\\
    & \qquad + 15  \sqrt{ \log \frac{320}{\delta} } + 184 \log \left(\frac{896}{\delta} \left[ \log \left( 2 + \frac{16}{\tau}\right) + 1 \right]^2 \right) \Bigg]
\end{align*}
\end{proof}

\section{Technical concentration bounds}\label{sec:concentration}
In this section we collect some technical results on concentration martingale difference sequences. These results are not new, and the proofs are for the most part exercises in known techniques (e.g. \cite{howard2021time,balsubramani2014sharp}), but we cannot find simple explicit statements of exactly the forms we require in the literature, so we provide them here and include proofs for completeness. Our approach follows the martingale mixture method used by \cite{balsubramani2014sharp}. We make no effort to achieve optimal constants, in several cases explicitly choosing weaker bounds to make the final numbers less complicated.

To start, we recall the notion of a sub-exponential martingale difference sequence (MDS) as follows:
\defsubexp*

We have the following useful way to obtain sub-exponential tails:

\begin{restatable}{Proposition}{propboundtosubexp}\label{prop:boundtosubexp}Suppose $\{X_t,\mathcal{F}_t\}$ is a MDS such that $\mathbb{E}[X_t^2|\mathcal{F}_t]\le \sigma_t^2$ and $|X_t|\le b_t$ almost everywhere for all $t$ for some sequence of random variable $\{\sigma_t,  b_t\}$ such that $\sigma_t,b_t$ is $\mathcal{F}_{t-1}$-measurable. Then $X_t$ is $(\sigma_t,2b_t)$ sub-exponential.
\end{restatable}

The main results of this section are the following two Theorems. First for scalar random variables we have the following one-sided concentration bound:

\begin{restatable}{Theorem}{thmmdsconcentration}\label{thm:mdsconcentration}
Suppose $\{X_t, \mathcal{F}_t\}$ is a $(\sigma_t,b_t)$ sub-exponential martingale difference sequence. Let $\sigmaunit$ be an arbitrary constant. Then with probability at least $1-\delta$, for all $t$ it holds that:
\begin{align*}
\sum_{i=1}^t X_i &\le 2\sqrt{\sum_{i=1}^{t}\sigma_i^2  \log\left(\frac{4}{\delta}\left[\log\left(\left[\sqrt{\sum_{i=1}^{t}\sigma_i^2/2\sigmaunit^2}\right]_1\right)+2\right]^2\right)}\\
    &\qquad +8\max(\sigmaunit,\max_{i\le t} b_i)\log\left(\frac{28}{\delta}\left[\log\left(\frac{\max(\sigmaunit,\max_{i\le t} b_i)}{\sigmaunit}\right)+2\right]^2\right)
\end{align*}
where $[x]_1=\max(1,x)$.
\end{restatable}

Next, for vector valued random variables, we have the following bound:
\begin{restatable}{Theorem}{thmvectorconcentration}\label{thm:vectorconcentration}
Suppose that $\{X_t,\mathcal{F}_t\}$ is  a vector-valued martingale difference sequence such that $\E[\|X_t\|^2|\mathcal{F}_{t-1}]\le \sigma_{t}^2$ and $\|X_t\|\le b_{t}$ almost everywhere for some sequence $\{\sigma_t, b_t\}$ such that $\sigma_t,b_t$ is $\mathcal{F}_{t-1}$-measurable. Let $\sigmaunit\ge 0$ be an arbitrary constant. Then with  probability at least $1-\delta$, for all $t$ we have:
\begin{align*}
    \left\|\sum_{i=1}^t X_i\right\|&\le 5\sqrt{\sum_{i=1}^t \sigma_i^2\log\left(\frac{16}{\delta}\left[\log\left(\left[\sqrt{\sum_{i=1}^{t}\sigma_i^2/\sigmaunit^2}\right]_1\right) +2\right]^2\right)}\\
    &\qquad +23 \max(\sigmaunit,\max_{i\le t}b_{i})\log\left(\frac{224}{\delta}\left[\log\left(\frac{2\max(\sigmaunit,\max_{i\le t} b_{i})}{\sigmaunit}\right)+2\right]^2\right)
\end{align*}
where $[x]_1=\max(1,x)$.
\end{restatable}

\subsection{Time-Uniform Concentration of sums of sub-exponential MDS}
In this section, we will prove Theorem~\ref{thm:mdsconcentration}.
\thmmdsconcentration*
\begin{proof}
First, observe that by replacing $b_t$ with $\max_{i\le t} b_i$, we may assume $ b_t\ge b_{t-1}$ for all $t$ with probability 1. Notice that after this operation, $X_t$ is still $(\sigma_t,b_t)$ sub-exponential. Under this assumption, it suffices to prove  the result with $b_t$ in place of $\max_{i\le t} b_i$.

Define $M_t = \sum_{i=1}^t X_i$ and let $\pi(\eta)$ be a to-be-specified probability density function on $\R$. Define:
$$
Z_t = \int_0^{1/\max(\sigmaunit, b_t)} \pi(\eta)\exp\left[\eta M_t - \frac{\eta^2 \sum_{i=1}^t \sigma_i^2}{2}\right]\ d\eta
$$
Notice that $Z_0\le \int_0^\infty \pi(\eta)\ d\eta=1$. We claim that $Z_t$ is itself a supermartingale adapted to the same filtration $\mathcal{F}_t$:

\begin{align*}
\mathbb{E}\left[Z_t|\mathcal{F}_{t-1}\right] &= \mathbb{E}\left[\int_0^{1/\max(\sigmaunit, b_t)} \pi(\eta)\exp\left[\eta M_t - \frac{\eta^2 \sum_{i=1}^{t} \sigma_i^2}{2}\right]\ d\eta|\ \mathcal{F}_{t-1}\right]\\
&=\mathbb{E}\left[\int_0^{1/\max(\sigmaunit, b_t)} \pi(\eta)\exp\left[\eta M_{t-1} - \frac{\eta^2 \sum_{i=1}^{t-1} \sigma_i^2}{2}\right]\exp\left[\eta X_t - \frac{\eta^2 \sigma_t^2}{2}\right]\ d\eta|\ \mathcal{F}_{t-1}\right]\\
&= \int_0^{1/\max(\sigmaunit,b_t)} \pi(\eta)\exp\left[\eta M_{t-1} - \frac{\eta^2 \sum_{i=1}^{t-1} \sigma_i^2}{2}\right]\mathbb{E}\left[\exp\left[\eta X_t - \frac{\eta^2 \sigma_{t}^2}{2}\right]|\ \mathcal{F}_{t-1}\right]\ d\eta
\intertext{Use the sub-exponentiality of $X_t$:}
&\le \int_0^{1/\max(\sigmaunit, b_t)} \pi(\eta)\exp\left[\eta M_{t-1} - \frac{\eta^2 \sum_{i=1}^{t-1} \sigma_i^2}{2}\right]\ d\eta
\intertext{Use $b_t\ge b_{t-1}$:}
&\le \int_0^{1/\max(\sigmaunit, b_{t-1})} \pi(\eta)\exp\left[\eta M_{t-1} - \frac{\eta^2 \sum_{i=1}^{t-1} \sigma_i^2}{2}\right]\ d\eta\\
&=Z_{t-1}
\end{align*}

Therefore, by Ville's maximal inequality \citep{ville1939etude}, we have that for all $\delta>0$:
$$
P\left[\sup_t Z_t\ge 1/\delta\right]\le \delta Z_0 = \delta
$$
Put another way, with probability at least $1-\delta$, $Z_t\le 1/\delta$ for all $t$.  

Now, let us define $\pi(\eta)$. With  the benefit of foresight, we choose a density on $[0,1/\sigmaunit]$:

\begin{align}
    \pi(\eta) = \frac{1}{ \eta(\log(1/\eta\sigmaunit) +2)^2}
\end{align}
We have
\begin{align}
    \pi(\eta) = \frac{d}{d\eta}\frac{1}{\log(1/\eta\sigmaunit)+2}
\end{align}
\begin{align}
    \frac{d\pi(\eta)}{d\eta} = -\frac{\log(1/\eta \sigmaunit)}{\eta^2 \left(\log(1/\eta\sigmaunit)+2\right)^3}
\end{align}
so that $\int_0^{1/\sigmaunit}\pi(\eta)\ d\eta=1$ and $\pi(\eta)$ is decreasing on $[0,1/\sigmaunit]$.

Next, for any given $\eta_\star\in[0,1/\sigmaunit]$, for all $K\ge 1$, for all $\eta\in[\eta_\star/K,\eta_\star]$, we have:
\begin{align*}
    \int_{\eta_\star/K}^{\eta_\star} \pi(\eta)\ d\eta  &\ge\pi(\eta_\star) \frac{K-1}{K}\eta_\star\\
    &\ge \frac{(K-1)}{K(\log(1/\eta_\star\sigmaunit)+2)^2}
\end{align*}
Further, for all $\eta\in[\eta_\star/K, \eta_\star]$, if $M_t\ge 0$ we also have:
\begin{align*}
    \eta M_t - \frac{\eta^2\sum_{i=0}^{t-1}\sigma_i^2}{2}\ge \frac{\eta_\star M_t}{K}  - \frac{\eta_\star^2\sum_{i=0}^{t-1}\sigma_i^2}{2}
\end{align*}
and otherwise of course we have $M_t\le 0$.

Combining these observations, we have that with probability at least $1-\delta$, for all $\mathcal{F}_{t-1}$-measurable $\eta_\star$ satisfying $\eta_\star\le 1/\max(\sigmaunit, b_t)$ and all $K\ge  1$, either $M_t\le 0$ or
\begin{align*}
    \frac{(K-1)}{K(\log(1/\eta_\star\sigmaunit)+2)^2}\exp\left(\frac{\eta_\star M_t}{K}  - \frac{\eta_\star^2\sum_{i=1}^{t}\sigma_i^2}{2}\right)&\le \int_0^{1/b_{t}}\pi(\eta) \exp\left[\eta M_t - \frac{\eta^2 \sum_{i=1}^{t} \sigma_i^2}{2}\right]\ d\eta\le \delta
\end{align*}
Now, rearranging this identity implies:
\begin{align}
    M_t&\le \log\left(\frac{K}{(K-1)\delta}\left[\log\left(\frac{1}{\eta_\star\sigmaunit}\right)+2\right]^2\right)\frac{K}{\eta_\star} + \frac{K\eta_\star\sum_{i=1}^{t}\sigma_i^2}{2}\label{eqn:Mbound}
\end{align}
So that overall we may discard the $M_t \le 0$ case as it is strictly weaker than the  above.

Now, again with the benefit of foresight, let us select 
\begin{align*}
    \eta_\star =\min\left( \frac{1}{\max(\sigmaunit, b_t)},\ \sqrt{\frac{    2 \log\left[ \frac{K}{(K-1)\delta} \left[\log\left(\left[\sqrt{\sum_{i=1}^{t}\sigma_i^2/2\sigmaunit^2}\right]_1\right)+2\right]^2 \right]    }{\sum_{i=1}^{t} \sigma_i^2}}\right)
\end{align*}
where $[x]_1=\max(1,x)$. Notice that $\eta_\star$ is $\mathcal{F}_{t-1}$-measurable since $b_t$ and $\sigma_1,\dots,\sigma_t$ are $\mathcal{F}_{t-1}$-measurable, and $\eta_\star \in[0, 1/\sigmaunit]$ with probability 1.

Now, to analyze the expression (\ref{eqn:Mbound}), we will consider both cases of the above minimum. First, let us assume
\begin{align*}
    1/\max(\sigmaunit, b_{t})\ge \sqrt{\frac{    2 \log\left[ \frac{K}{(K-1)\delta} \left[\log\left(\left[\sqrt{\sum_{i=1}^{t}\sigma_i^2/2\sigmaunit^2}\right]_1\right)+2\right]^2 \right]    }{\sum_{i=1}^{t} \sigma_i^2}}
\end{align*}
Then we can bound $\eta_\star$:
\begin{align*}
    \eta_\star &\ge \sqrt{\frac{    2 \log\left[ \frac{K}{(K-1)\delta} \left[2\right]^2 \right]    }{\sum_{i=1}^{t} \sigma_i^2}}\\
    &\ge  \sqrt{\frac{    2 }{\sum_{i=1}^{t} \sigma_i^2}}
\end{align*}
Therefore:
\begin{align*}
    \log\left(\frac{K}{(K-1)\delta}\left[\log\left(\frac{1}{\eta_\star\sigmaunit}\right)+2\right]^2\right)&\le  \log\left(\frac{K}{(K-1)\delta}\left[\log\left(\left[\sqrt{\sum_{i=0}^{t-1}\sigma_i^2/2\sigmaunit^2}\right]_1\right)+2\right]^2\right)
\end{align*}
from which we conclude:
\begin{align*}
    M_t&\le K\sqrt{2\sum_{i=1}^{t}\sigma_i^2  \log\left(\frac{K}{(K-1)\delta}\left[\log\left(\left[\sqrt{\sum_{i=1}^{t}\sigma_i^2/2\sigmaunit^2}\right]_1\right)+2\right]^2\right)}
\end{align*}

Now, on the other hand let us suppose that 
\begin{align*}
    1/\max(\sigmaunit, b_{t})< \sqrt{\frac{    2 \log\left[ \frac{K}{(K-1)\delta} \left[\log\left(\left[\sqrt{\sum_{i=1}^{t}\sigma_i^2/2\sigmaunit^2}\right]_1\right)+2\right]^2 \right]    }{\sum_{i=1}^{t} \sigma_i^2}}
\end{align*}
This implies:
\begin{align*}
    \sum_{i=1}^{t}\sigma_i^2 &< 2\min(\sigmaunit^2, b_{t}^2)    \log\left[ \frac{K}{(K-1)\delta} \left[\log\left(\left[\sqrt{\sum_{i=1}^{t}\sigma_i^2/2\sigmaunit^2}\right]_1\right)+2\right]^2 \right] 
\end{align*}
Therefore, from Lemma~\ref{lem:selfbound} we have:
\begin{align*}
\sum_{i=1}^{t}\sigma_i^2&\le 8\max(\sigmaunit^2, b_{t}^2)    \log\left[ \frac{4\sqrt{K}}{\sqrt{(K-1)\delta}}\log\left(e+16\frac{\max(\sigmaunit^2, b_{t}^2)}{\sigmaunit^2}\right) \right]\\
    &\le 8\max(\sigmaunit^2, b_{t}^2)    \log\left[ \frac{4\sqrt{K}}{\sqrt{(K-1)\delta}}\log\left(20\frac{\max(\sigmaunit^2, b_{t}^2)}{\sigmaunit^2}\right) \right] 
\end{align*}
In this case, we will have $\eta_\star = 1/\max(\sigmaunit, b_{t})$ and so obtain:
\begin{align*}
    M_t&\le K\max(\sigmaunit,b_{t})\log\left(\frac{K}{(K-1)\delta}\left[\log\left(\frac{\max(\sigmaunit,b_{t})}{\sigmaunit}\right)+2\right]^2\right) \\
    &\qquad + 4\max(\sigmaunit, b_{t})    \log\left[ \frac{4\sqrt{K}}{\sqrt{(K-1)\delta}}\log\left(20\frac{\max(\sigmaunit^2, b_{t}^2)}{\sigmaunit^2}\right) \right] \\
    &\le K\max(\sigmaunit,b_{t})\log\left(\frac{K}{(K-1)\delta}\left[\log\left(\frac{\max(\sigmaunit,b_{t})}{\sigmaunit}\right)+2\right]^2\right) \\
    &\qquad + 4\max(\sigmaunit, b_{t})    \log\left[ \frac{8\sqrt{K}}{\sqrt{(K-1)\delta}}\log\left(\frac{5\max(\sigmaunit, b_{t-1})}{\sigmaunit}\right) \right] \\
    &\le K\max(\sigmaunit,b_{t})\log\left(\frac{K}{(K-1)\delta}\left[\log\left(\frac{\max(\sigmaunit,b_{t})}{\sigmaunit}\right)+2\right]^2\right) \\
    &\qquad + 4\max(\sigmaunit, b_{t})    \log\left[ \frac{8\sqrt{K}}{\sqrt{(K-1)\delta}}\left[\log\left(\frac{\max(\sigmaunit, b_{t})}{\sigmaunit}\right) +2\right]\right]\\
    &\le 5K\max(\sigmaunit,b_{t})\log\left(\frac{8K}{(K-1)\delta}\left[\log\left(\frac{\max(\sigmaunit,b_{t})}{\sigmaunit}\right)+2\right]^2\right)
\end{align*}

Finally, let us set $K=\sqrt{2}$ to obtain:
\begin{align*}
    M_t&\le K\sqrt{2\sum_{i=1}^{t}\sigma_i^2  \log\left(\frac{K}{(K-1)\delta}\left[\log\left(\left[\sqrt{\sum_{i=1}^{t}\sigma_i^2/2\sigmaunit^2}\right]_1\right)+2\right]^2\right)}\\
    &\qquad +5K\max(\sigmaunit,b_{t})\log\left(\frac{8K}{(K-1)\delta}\left[\log\left(\frac{\max(\sigmaunit,b_{t-1})}{\sigmaunit}\right)+2\right]^2\right)\\
    &\le 2\sqrt{\sum_{i=1}^{t}\sigma_i^2  \log\left(\frac{4}{\delta}\left[\log\left(\left[\sqrt{\sum_{i=1}^{t}\sigma_i^2/2\sigmaunit^2}\right]_1\right)+2\right]^2\right)}\\
    &\qquad +8\max(\sigmaunit,b_{t})\log\left(\frac{28}{\delta}\left[\log\left(\frac{\max(\sigmaunit,b_{t})}{\sigmaunit}\right)+2\right]^2\right)
\end{align*}

\end{proof}

\subsection{Bounds on Sums of Squares}

It is also often useful to bound sums of the form  $\sum_{i=1}^t Z_i^2$ for some sequence $Z_i$. Here we collect a useful bound:
\begin{Theorem}\label{thm:sumsquares}
Suppose $\{Z_i\}$ is a sequence of random variables adapted to a filtration $\{\mathcal{F}_t\}$. Further, suppose $\mathbb{E}[Z_i^2]\le \sigma_i^2$ and $|Z_i|\le b_{i}$ for all $i$ with probability 1 for some $\sigma_i$ and $b_i$ for a sequence $\{\sigma_i,b_i\}$ such that $\sigma_i$ and $b_i$ are $\mathcal{F}_{i-1}$-measurable. Then for any $\sigmaunit>0$, with probability at least $1-\delta$ for all $t$:

\begin{align*}
\sum_{i=1}^t Z_i^2&\le 3\sum_{i=1}^t \sigma_i^2\log\left(\frac{4}{\delta}\left[\log\left(\left[\sqrt{\sum_{i=1}^{t}\sigma_i^2/\sigmaunit^2}\right]_1\right) +2\right]^2\right) \\
    &\qquad +20\max(\sigmaunit^2,\max_{i\le t}b_{i}^2)\log\left(\frac{112}{\delta}\left[\log\left(\frac{2\max(\sigmaunit,\max_{i\le t} b_{i})}{\sigmaunit}\right)+1\right]^2\right)
\end{align*}
\end{Theorem}

\begin{proof}
Define $X_t = Z_t^2 - \mathbb{E}[Z_t^2|\mathcal{F}_{t-1}]$ so that $\{X_t,\mathcal{F}_t\}$ is a MDS. Further, notice that $|X_t|\le b_t^2$ for all $t$ with probability 1 and

\begin{align*}
\mathbb{E}[X_t^2|\mathcal{F}_{t-1}]&\le b_t^2\mathbb{E}[|Z_t^2 - \mathbb{E}[Z_t^2|\mathcal{F}_{t-1}]|\ |\ \mathcal{F}_{t-1}]\\
&\le 2b_t^2\sigma_t^2
\end{align*}

Therefore, by Proposition~\ref{prop:boundtosubexp}, $X_t$ is $(\sqrt{2b_t^2\sigma_t^2}, 2b_t^2)$ sub-exponential. 

Thus, by our time-uniform concentration bound (Theorem~\ref{thm:mdsconcentration}), for any $\sigmaunit>0$, with probability at least $1-\delta$ we have:
\begin{align*}
\sum_{i=1}^t X_i &\le 2\sqrt{2\sum_{i=1}^{t}\sigma_i^2 b_i^2 \log\left(\frac{4}{\delta}\left[\log\left(\left[\sqrt{\sum_{i=1}^{t}\sigma_i^2b_i^2/\sigmaunit^4}\right]_1\right)+2\right]^2\right)}\\
    &\qquad +8\max(\sigmaunit^2,2\max_{i\le t} b_{i}^2)\log\left(\frac{28}{\delta}\left[\log\left(\frac{\max(\sigmaunit^2,2\max_{i\le t} b_{t}^2)}{\sigmaunit^2}\right)+2\right]^2\right)\\
    &\le  2\sqrt{2\max_{i\le t} b_{i}^2\sum_{i=1}^{t}\sigma_i^2 \log\left(\frac{4}{\delta}\left[\log\left(\left[\sqrt{\sum_{i=1}^{t}\sigma_i^2b_i^2/\sigmaunit^4}\right]_1\right)+2\right]^2\right)}\\
    &\qquad +8\max(\sigmaunit^2,2\max_{i\le t}b_{i}^2)\log\left(\frac{28}{\delta}\left[\log\left(\frac{\max(\sigmaunit^2,2\max_{i\le t} b_{i}^2)}{\sigmaunit^2}\right)+2\right]^2\right)
\end{align*}

Now, by Young inequality:
\begin{align*}
    &2\max_{i\le t}b_i^2\sum_{i=1}^t \sigma_i^2\log\left(\frac{4}{\delta}\left[\log\left(\left[\sqrt{\sum_{i=1}^{t}\sigma_i^2b_i^2/\sigmaunit^4}\right]_1\right)+2\right]^2\right) \\
    &\qquad\qquad\le\left(\max_{i\le t}b_i^2\log\left(\frac{28}{\delta}\left[\log\left(\frac{\max(\sigmaunit^2,2\max_{i\le t} b_{i}^2)}{\sigmaunit^2}\right)+2\right]^2\right)\right)^2\\
    &\qquad\qquad\qquad + \left(\sum_{i=1}^t \sigma_i^2\frac{\log\left(\frac{4}{\delta}\left[\log\left(\left[\sqrt{\sum_{i=1}^{t}\sigma_i^2b_i^2/\sigmaunit^4}\right]_1\right)+2\right]^2\right)}{\log\left(\frac{28}{\delta}\left[\log\left(\frac{\max(\sigmaunit^2,2\max_{i\le t} b_{i}^2)}{\sigmaunit^2}\right)+2\right]^2\right)}\right)^2
\end{align*}
To simplify this expression, we consider the following identity:
\begin{align*}
    \frac{\log\left(\frac{4}{\delta}\left[\log\left(\left[\sqrt{\sum_{i=1}^{t}\sigma_i^2b_i^2/\sigmaunit^4}\right]_1\right)+2\right]^2\right)}{\log\left(\frac{28}{\delta}\left[\log\left(\frac{\max(\sigmaunit^2,2\max_{i\le t} b_{i}^2)}{\sigmaunit^2}\right)+2\right]^2\right)}&\le \frac{\log\left(\frac{4}{\delta}\left[\log\left(\left[\sqrt{\sum_{i=1}^{t}\sigma_i^2/\sigmaunit^2}\sqrt{\max_{i\le t}b_i^2/\sigmaunit^2}\right]_1\right)+2\right]^2\right)}{\log\left(\frac{28}{\delta}\left[\log\left(\frac{\max(\sigmaunit^2,2\max_{i\le t} b_{i}^2)}{\sigmaunit^2}\right)+2\right]^2\right)}\\
    &\le \frac{\log\left(\frac{4}{\delta}\left[\log\left(\left[\sqrt{\sum_{i=1}^{t}\sigma_i^2/\sigmaunit^2}\right]_1\right) +\frac{1}{2}\log\left(\frac{\max(\sigmaunit^2, 2\max_{i\le t}b_i^2)}{\sigmaunit^2}\right)+2\right]^2\right)}{\log\left(\frac{28}{\delta}\left[\log\left(\frac{\max(\sigmaunit^2,2\max_{i\le t} b_{i}^2)}{\sigmaunit^2}\right)+2\right]^2\right)}
\end{align*}
Now, consider two cases, either 
\begin{align*}
    \log\left(\left[\sqrt{\sum_{i=1}^{t}\sigma_i^2/\sigmaunit^2}\right]_1\right) \le \frac{1}{2}\log\left(\frac{\max(\sigmaunit^2, 2\max_{i\le t}b_i^2)}{\sigmaunit^2}\right)
\end{align*}
or not. In the former case, we have:
\begin{align*}
    \frac{\log\left(\frac{4}{\delta}\left[\log\left(\left[\sqrt{\sum_{i=1}^{t}\sigma_i^2b_i^2/\sigmaunit^4}\right]_1\right)+2\right]^2\right)}{\log\left(\frac{28}{\delta}\left[\log\left(\frac{2\max(\sigmaunit^2,2\max_{i\le t} b_{i}^2)}{\sigmaunit^2}\right)+2\right]^2\right)}&\le\frac{\log\left(\frac{4}{\delta}\left[\log\left(\left[\sqrt{\sum_{i=1}^{t}\sigma_i^2/\sigmaunit^2}\right]_1\right) +\frac{1}{2}\log\left(\frac{\max(\sigmaunit^2, 2\max_{i\le t}b_i^2)}{\sigmaunit^2}\right)+2\right]^2\right)}{\log\left(\frac{28}{\delta}\left[\log\left(\frac{\max(\sigmaunit^2,2\max_{i\le t} b_{i}^2)}{\sigmaunit^2}\right)+2\right]^2\right)}\\
    &\le \frac{\log\left(\frac{4}{\delta}\left[\log\left(\frac{\max(\sigmaunit^2, 2\max_{i\le t}b_i^2)}{\sigmaunit^2}\right)+2\right]^2\right)}{\log\left(\frac{28}{\delta}\left[\log\left(\frac{\max(\sigmaunit^2,2\max_{i\le t} b_{i}^2)}{\sigmaunit^2}\right)+2\right]^2\right)}\\
    &\le 1\\
    &\le \log\left(\frac{4}{\delta}\left[\log\left(\left[\sqrt{\sum_{i=1}^{t}\sigma_i^2/\sigmaunit^2}\right]_1\right)+2\right]^2\right)
\end{align*}
While in the latter case,
\begin{align*}
    \frac{\log\left(\frac{4}{\delta}\left[\log\left(\left[\sqrt{\sum_{i=1}^{t}\sigma_i^2b_i^2/\sigmaunit^4}\right]_1\right)+2\right]^2\right)}{\log\left(\frac{28}{\delta}\left[\log\left(\frac{\max(\sigmaunit^2,2\max_{i\le t} b_{i}^2)}{\sigmaunit^2}\right)+2\right]^2\right)}&\le\frac{\log\left(\frac{4}{\delta}\left[2\log\left(\sqrt{\sum_{i=1}^{t}\sigma_i^2/\sigmaunit^2}\right) +2\right]^2\right)}{\log\left(\frac{28}{\delta}\left[\log\left(\frac{\max(\sigmaunit^2,2\max_{i\le t} b_{i}^2)}{\sigmaunit^2}\right)+2\right]^2\right)}\\
    &\le \log\left(\frac{4}{\delta}\left[\log\left(\left[\sqrt{\sum_{i=1}^{t}\sigma_i^2/\sigmaunit^2}\right]_1\right) +2\right]^2\right)
\end{align*}
So in both cases, we have
\begin{align*}
    \frac{\log\left(\frac{4}{\delta}\left[\log\left(\left[\sqrt{\sum_{i=1}^{t}\sigma_i^2b_i^2/\sigmaunit^4}\right]_1\right)+2\right]^2\right)}{\log\left(\frac{28}{\delta}\left[\log\left(\frac{\max(\sigmaunit^2,2\max_{i\le t} b_{i}^2)}{\sigmaunit^2}\right)+2\right]^2\right)}&\le\log\left(\frac{4}{\delta}\left[\log\left(\left[\sqrt{\sum_{i=1}^{t}\sigma_i^2/\sigmaunit^2}\right]_1\right) +2\right]^2\right)
\end{align*}
Therefore,
\begin{align*}
    &2\max_{i\le t}b_i^2\sum_{i=1}^t \sigma_i^2\log\left(\frac{4}{\delta}\left[\log\left(\left[\sqrt{\sum_{i=1}^{t}\sigma_i^2b_i^2/\sigmaunit^4}\right]_1\right)+2\right]^2\right)\\ &\qquad\qquad\le\left(\max_{i\le t}b_i^2\log\left(\frac{28}{\delta}\left[\log\left(\frac{\max(\sigmaunit^2,2\max_{i\le t} b_{i}^2)}{\sigmaunit^2}\right)+2\right]^2\right)\right)^2\\
    &\qquad\qquad\qquad + \left(\sum_{i=1}^t \sigma_i^2\log\left(\frac{4}{\delta}\left[\log\left(\left[\sqrt{\sum_{i=1}^{t}\sigma_i^2/\sigmaunit^2}\right]_1\right) +2\right]^2\right)\right)^2
\end{align*}

Combining this with the identity $\sqrt{a+b}\le \sqrt{a}+\sqrt{b}$, we have:

\begin{align*}
\sum_{i=1}^t X_i &\le2\sqrt{2\max_{i\le t} b_{i}^2\sum_{i=1}^{t}\sigma_i^2 \log\left(\frac{4}{\delta}\left[\log\left(\left[\sqrt{\sum_{i=1}^{t}\sigma_i^2b_i^2/\sigmaunit^4}\right]_1\right)+2\right]^2\right)}\\
    &\qquad +8\max(\sigmaunit^2,2\max_{i\le t}b_{i}^2)\log\left(\frac{28}{\delta}\left[\log\left(\frac{\max(\sigmaunit^2,2\max_{i\le t} b_{i}^2)}{\sigmaunit^2}\right)+2\right]^2\right)\\
    &\le 2\sum_{i=1}^t \sigma_i^2\log\left(\frac{4}{\delta}\left[\log\left(\left[\sqrt{\sum_{i=1}^{t}\sigma_i^2/\sigmaunit^2}\right]_1\right) +2\right]^2\right) \\
    &\qquad +10\max(\sigmaunit^2,2\max_{i\le t}b_{i}^2)\log\left(\frac{28}{\delta}\left[\log\left(\frac{\max(\sigmaunit^2,2\max_{i\le t} b_{i}^2)}{\sigmaunit^2}\right)+2\right]^2\right)\\
    &\le 2\sum_{i=1}^t \sigma_i^2\log\left(\frac{4}{\delta}\left[\log\left(\left[\sqrt{\sum_{i=1}^{t}\sigma_i^2/\sigmaunit^2}\right]_1\right) +2\right]^2\right) \\
    &\qquad +20\max(\sigmaunit^2,\max_{i\le t}b_{i}^2)\log\left(\frac{112}{\delta}\left[\log\left(\frac{2\max(\sigmaunit,\max_{i\le t} b_{i})}{\sigmaunit}\right)+1\right]^2\right)
\end{align*}

Finally, observe that
\begin{align*}
    \sum_{i=1}^t Z_i^2 &=\sum_{i=1}^t X_i + \E[Z_i^2|\mathcal{F}_{t-1}]\\
    &\le \sum_{i=1}^t \sigma_i^2 + \sum_{i=1}^t X_i\\
    &\le 3\sum_{i=1}^t \sigma_i^2\log\left(\frac{4}{\delta}\left[\log\left(\left[\sqrt{\sum_{i=1}^{t}\sigma_i^2/\sigmaunit^2}\right]_1\right) +1\right]^2\right) \\
    &\qquad +20\max(\sigmaunit^2,\max_{i\le t}b_{i}^2)\log\left(\frac{112}{\delta}\left[\log\left(\frac{2\max(\sigmaunit,\max_{i\le t} b_{i})}{\sigmaunit}\right)+1\right]^2\right)
\end{align*}

\end{proof}

\subsection{From scalar to vector (Hilbert space) concentration}

In this section we extend our results to concentration of norm of vectors in Hilbert space. The technique follows that of \cite{cutkosky2021high}, which makes use of a particular scalar sequence associated to any vector sequence described by \cite{cutkosky2018algorithms}. Given any sequence $X_1,X_2,\dots$ of vectors, define a sequence $s_1,s_2,\dots$ of  scalars as follows:
\begin{enumerate}
\item $s_0=0$.
\item If $\sum_{i=1}^{t-1} X_i\ne 0$, set
$$
s_t = \text{sign}\left(\sum_{i=1}^{t-1} s_i\right)\frac{\left\langle \sum_{i=1}^{t-1} X_i, X_t\right\rangle}{\left\|\sum_{i=1}^{t-1} X_i\right\|}
$$
where we define $\text{sign}(z) = 1$ if $z\ge 0$ and $\text{sign}(z)=-1$ otherwise.
\item If $\sum_{i=1}^{t-1} X_i=0$, set $s_t=0$.
\end{enumerate}
Clearly if $\{X_t\}$ is a random sequence adapted to the filtration $\mathcal{F}_t$, then so is $s_t$. Now, these $s_t$ have the following interesting property:
\begin{Lemma}[\cite{cutkosky2021high}, Lemma 10]\label{lem:1dconstruction}
For all $t$, we have $|s_t|\le \|X_t\|$ and

\begin{align*}
\left\|\sum_{i=1}^t X_i\right\|\le \left|\sum_{i=1}^t s_i\right| + \sqrt{\max_{i\le t}\|X_i\|^2 + \sum_{i=1}^t \|X_i\|^2}
\end{align*}
\end{Lemma}

Now, we need to use this result. The key is that if $X_t$ is a MDS, then the $s_t$ will be also. Thus, we can bound the sum of the $X_t$ using the sum of the $s_t$, which can in turn be bounded by \emph{scalar} martingale concentration bounds. Let us instantiate this using our previous  bounds to obtain:

\thmvectorconcentration*


\begin{proof}

Observe from the construction  of the $s_t$ sequence that $\{s_t\}$ is a MDS adapted to $\mathcal{F}_t$, and that $\E[s_t^2\ |\ \mathcal{F}_{t-1}]\le \sigma_{t-1}^2$ and $|s_t|\le b_{t-1}$. Therefore $s_t$ is $\sigma_t, 2b_t$ subgaussian. Invoking Theorem~\ref{thm:mdsconcentration} (with a union bound for a two-sided inequality), with probability at least $1-\delta/2$ we have:
\begin{align*}
    \left|\sum_{i=1}^t  s_i\right|&\le 2\sqrt{\sum_{i=1}^{t}\sigma_i^2  \log\left(\frac{16}{\delta}\left[\log\left(\left[\sqrt{\sum_{i=1}^{t}\sigma_i^2/2\sigmaunit^2}\right]_1\right)+2\right]^2\right)}\\
    &\qquad +16\max(\sigmaunit,\max_{i\le t} b_i)\log\left(\frac{112}{\delta}\left[\log\left(\frac{2\max(\sigmaunit,\max_{i\le t} b_i)}{\sigmaunit}\right)+2\right]^2\right)
\end{align*}
Next, observe that $\|X_t\|$ satisfies the conditions of Theorem~\ref{thm:sumsquares} so that also with probability at least $1-\delta/2$:
\begin{align*}
    \sum_{i=1}^t \|X_i\|^2&\le3\sum_{i=1}^t \sigma_i^2\log\left(\frac{8}{\delta}\left[\log\left(\left[\sqrt{\sum_{i=1}^{t}\sigma_i^2/\sigmaunit^2}\right]_1\right) +2\right]^2\right) \\
    &\qquad +20\max(\sigmaunit^2,\max_{i\le t}b_{i}^2)\log\left(\frac{224}{\delta}\left[\log\left(\frac{2\max(\sigmaunit,\max_{i\le t} b_{i})}{\sigmaunit}\right)+1\right]^2\right)
\end{align*}

Putting this together with  Lemma~\ref{lem:1dconstruction} we  have
\begin{align*}
    \left\|\sum_{i=1}^t X_i\right\|&\le \left|\sum_{i=1}^{t} s_i\right| + \sqrt{\max_{i\le t-1}\|X_i\|^2 + \sum_{i=1}^{t} \|X_i\|^2}\\
    &\le \left|\sum_{i=1}^{t} s_i\right| + \sqrt{2\sum_{i=1}^{t} \|X_i\|^2}\\
    &\le 2\sqrt{\sum_{i=1}^{t}\sigma_i^2  \log\left(\frac{16}{\delta}\left[\log\left(\left[\sqrt{\sum_{i=1}^{t}\sigma_i^2/2\sigmaunit^2}\right]_1\right)+2\right]^2\right)}\\
    &\qquad +16\max(\sigmaunit,\max_{i\le t} b_i)\log\left(\frac{112}{\delta}\left[\log\left(\frac{2\max(\sigmaunit,\max_{i\le t} b_i)}{\sigmaunit}\right)+2\right]^2\right)\\
    &\qquad+\sqrt{6\sum_{i=1}^t \sigma_i^2\log\left(\frac{8}{\delta}\left[\log\left(\left[\sqrt{\sum_{i=1}^{t}\sigma_i^2/\sigmaunit^2}\right]_1\right) +2\right]^2\right)} \\
    &\qquad +\sqrt{40\max(\sigmaunit^2,\max_{i\le t}b_{i}^2)\log\left(\frac{224}{\delta}\left[\log\left(\frac{2\max(\sigmaunit,\max_{i\le t} b_{i})}{\sigmaunit}\right)+1\right]^2\right)}\\
    &\le 5\sqrt{\sum_{i=1}^t \sigma_i^2\log\left(\frac{16}{\delta}\left[\log\left(\left[\sqrt{\sum_{i=1}^{t}\sigma_i^2/\sigmaunit^2}\right]_1\right) +2\right]^2\right)}\\
    &\qquad +23 \max(\sigmaunit,\max_{i\le t}b_{i})\log\left(\frac{224}{\delta}\left[\log\left(\frac{2\max(\sigmaunit,\max_{i\le t} b_{i})}{\sigmaunit}\right)+2\right]^2\right)
\end{align*}

\end{proof}
\subsection{Proof of Proposition~\ref{prop:boundtosubexp}}
\propboundtosubexp*
\begin{proof}
Suppose $\lambda \le 1/2b_t$. Then we compute for any $k\ge 2$:

\begin{align*}
\mathbb{E}\left[\frac{\lambda^kX_t^k}{k!}\ |\ \mathcal{F}_{t-1}\right]&\le \frac{\lambda^kb_t^{k-2}}{k!}\mathbb{E}[X_t^2| \mathcal{F}_{t-1}]\\
&\le \frac{\lambda^k b_t^{k-2}\sigma_t^2}{k!}\\
&\le \frac{\lambda^2\sigma_t^2}{2^{k-2}k!}
\end{align*}

Further, since $X_t$ is a MDS, we also have $\mathbb{E}\left[\lambda X_t\ |\ \mathcal{F}_{t-1}\right]=0$.
Therefore:
\begin{align*}
\mathbb{E}[\exp(\lambda X_t)|\mathcal{F}_{t-1}]&\le 1+\sum_{k=2}^\infty \frac{\lambda^2\sigma_t^2}{2^{k-2}k!}\\
&\le 1+\frac{\lambda^2 \sigma_t^2}{2}\sum_{i=0}^\infty \frac{1}{2^i}\\
&=1+ \frac{\lambda^2\sigma_t^2}{2}\\
&\le \exp(\lambda^2\sigma_t^2/2)
\end{align*}
where the last line uses the identity $1+x\le \exp(x)$.
\end{proof}

\subsection{Classical Concentration Bound for MDS}
\begin{Lemma}[Scaled Sub-exponential] \label{lemma:scaled_berinstein}
Suppose that $\{ X_t, \mathcal{F}_t \} $ is a MDS such that $\E[X_t|\mathcal{F}_t] = 0$, $\E[X_t^2 \mid \mathcal{F}_t] \leq \sigma$ and $|X_t|\leq b$ almost surely for some fixed $\sigma, b$.  Let $\nu$ be an arbitrary fixed number, then for all $t$ with probability at least $1-\delta$:
\begin{align*}
    \nu X_t \leq 2|\nu|b \log\frac{1}{\delta} + |\nu| \sigma \sqrt{ 2 \log \frac{1}{\delta}}
\end{align*}
\end{Lemma}
\begin{proof}
First we have $\E[\nu^2X_t^2]\leq \nu^2 \sigma^2, |\nu X_t| \leq |\nu|b$ almost surely and $\{ \nu X_t, \mathcal{F}_t \}$ is also a MDS. By Proposition \ref{prop:boundtosubexp}, $\nu X_t$ is $(|\nu|\sigma, 2|\nu|b)$ sub-exponential. Use definition \ref{def:subexp} and tower rule,
\begin{align*}
    \E[\exp(\lambda \nu X_t)] = \E[\E[\exp(\lambda \nu X_t) \mid \mathcal{F}_{t-1}]] = \exp(\lambda^2 |\nu| \sigma^2/2 )
\end{align*}
for $\lambda \leq 1/(2|\nu|b)$. The above inequality make us returns to the standard result of independent sub-exponential random variable.
\begin{align*}
    \E[\exp(\lambda \nu X_t)] = \exp(\lambda^2 |\nu| \sigma^2/2 )
\end{align*}
Hence, from the standard sub-exponential tails 
\begin{align*}
    P[\nu X_t \geq a] \leq 
    \begin{cases}
    \exp(-a^2/2\nu^2 \sigma^2) \quad 0 \leq a \leq \sigma^2|\nu|/2b, \\
    \exp(-a/2|\nu|b) \quad \quad a >\sigma^2|\nu|/2b
    \end{cases}
\end{align*}
Set above quantities as $\delta$ and rearrange for $a$:
\begin{align*}
    a = \max \left(2|\nu|b \log \frac{1}{\delta}, \sqrt{2|\nu|^2 \sigma^2 \log \frac{1}{\delta}} \right)
\end{align*}
Hence with probability at least $1-\delta$
\begin{align*}
    \nu X_t &\leq \max \left(2|\nu|b \log \frac{1}{\delta}, \sqrt{2|\nu|^2 \sigma^2 \log \frac{1}{\delta}} \right)\\
    & \leq 2|\nu|b \log \frac{1}{\delta} + |\nu| \sigma \sqrt{2 \log \frac{1}{\delta}}
\end{align*}
\end{proof}

\begin{Lemma}[Scaled Sub-exponential Sum] \label{lemma:scaled_sum_berinstein}
Suppose that $\{ X_t, \mathcal{F}_t \} $ is a MDS such that $\E[X_t \mid \mathcal{F}_t] = 0$, $\E[X_t^2 \mid \mathcal{F}_t] \leq \sigma$ and $|X_t|\leq b$ almost surely.  Let $\nu$ be an arbitrary fixed number, then for all $t$ with probability at least $1-\delta$:
\begin{align*}
   \sum_{i=1}^{t} \nu X_i \leq 2|\nu|b \log \frac{1}{\delta} + |\nu| \sigma \sqrt{2 t \log \frac{1}{\delta}} , \quad \forall t
\end{align*}
\end{Lemma}
\begin{proof}

\begin{align*}
    \E[e^{\lambda \sum_{i=1}^{t} \nu X_i }] &= \prod_{i=1}^{t} \E[e^{\lambda \nu X_i}]
    \intertext{for $|\lambda| \leq \frac{1}{2|\nu|b}$, invoke Lemma \ref{lemma:scaled_berinstein}}
    & \leq \exp( t \lambda^2 |\nu| \sigma^2/2 )\\
\end{align*}
Let $\sigma' = \sigma \sqrt{t}, \sigma' \in \R$, hence we can directly use the concentration bound derived in Lemma \ref{lemma:scaled_berinstein} to complete the proof.
\end{proof}

\begin{Lemma}[Sub-exponential Squared]\label{lemma:truncated_grad_concentration}
Suppose $\{X_t, \mathcal{F}_t \}$ is a MDS with $|X_t| \leq b$ and $\E[X_t^2 \mid \mathcal{F}_t] \leq \sigma^2$ almost surely for some fixed $\sigma, b$. Then with probability at least $1-\delta$
\begin{align*}
    \sum_{t=1}^{T} X_t^2 \leq \frac{3\sigma^2}{2}T + \frac{5}{3}b^2 \log \frac{1}{\delta}
\end{align*}
\end{Lemma}
\begin{proof}
Let $Z_t = X_t^2 - \E[X_t^2]$, $Z_0, \cdots, Z_T$ is a martingale difference sequence adapted to $\mathcal{F}_t$. Also  $|Z_t| < X_t^2 \leq b^2$. Also
\begin{align*}
    \E[Z_t^2] & = \E[|Z_t|\cdot |Z_t|] \leq b^2 \E[|Z_t|] \leq  b^2 \E[x_t^2]  \leq b^2 \sigma^2
\end{align*}
From Freedman's inequality for martingale sequences (see e.g. \cite{tropp2011freedman}, or Lemma 11 in \cite{cutkosky2021high} for the form we use here), with probability at least $1-\delta$,
\begin{align*}
    \sum_{t=1}^{T} Z_t \leq \frac{2}{3} b^2 \log \frac{1}{\delta} + \sigma b \sqrt{2T \log \frac{1}{\delta}}
\end{align*}
Rearranging the definition of $Z_t$, with probability at least $1-\delta$:
\begin{align*}
    \sum_{t=1}^{T} X_t^2 & = \sum_{t=1}^{T} Z_t + \sum_{t=1}^{T} \E[X_t^2]\\
    & \leq  \frac{2}{3} b^2 \log \frac{1}{\delta} + \sigma b \sqrt{2T \log \frac{1}{\delta}} + T \sigma^2 \numberthis \label{eqn:clip_grad_concentration}
\end{align*}
by young's inequality $\sigma b \leq \sigma^2/(2\lambda) + \lambda b^2 /2$, set $ \lambda = \sqrt{ 2 \log \frac{1}{\delta} / T }$, we complete the proof
\end{proof}

\subsection{Another Technical Lemma}

\begin{Lemma}\label{lem:selfbound}
Suppose $Z$ is such that
\begin{align*}
    Z&\le A \log\left(B\left[\log\left(\left[C\sqrt{Z}\right]_1\right)+2\right]^2\right)
\end{align*}
for some constants $A, B, C\ge 0$, where $[x]_1=\max(1,x)$ Then
\begin{align*}
    Z&\le  4A\log\left(4\sqrt{B} \log(e+16C^2A)\right)
\end{align*}
\end{Lemma}
\begin{proof}
Expanding the logarithms in the given bound on $Z$, we have:
\begin{align*}
    Z&\le A\log(B) + 2A\log(\log([C\sqrt{Z}]_1)+2)
\end{align*}
Now, since $a+b\le \max(2a,2b)$, we have that either $Z\le 2A\log(B)$, or
\begin{align*}
    Z&\le 4A\log(\log([C\sqrt{Z}]_1)+2)
\end{align*}
In the first case, we are done since $2A\log(B)\le 4A\log\left(4\sqrt{B} \log(e+16C^2A)\right)$, so let us consider only the  second case $Z\le 4A\log(\log([C\sqrt{Z}]_1)+2)$.
Now, we define the function 
\begin{align}
f(x) = 4A\log(\log([C\sqrt{x}]_1))+2)    
\end{align}
Notice that for all $x\ge A$, we have either $f'(x)=0$ or $C\sqrt{x}\ge 1$ and
\begin{align*}
    f'(x) = \frac{2A}{x\log(C\sqrt{x})+2x}< 1
\end{align*}
so that if $Z_\star$ is any value satisfying $Z_\star \ge A$ and $Z_\star \ge f(Z_\star)$, then we must have $Z\le Z_\star$: otherwise $f(Z)= f(Z_\star) +  \int_{Z_\star}^Z f'(z)dz < f(Z_\star)+ Z-Z_\star\le Z$,  a contradiction. Let us consider:
\begin{align*}
    Z_\star &=4A\log\left[4\log(e+QA)\right]
\end{align*}
for some to-be-specified $Q\ge 0$.
Notice that this $Z_\star$ clearly satisfies $Z_\star \ge 8A\log(2)$. Let us show that $Z_\star \ge f(Z_\star)$.

Again using $x+y\le 2\max(x,y)$, we have:
\begin{align*}
    f(Z_\star) &\le 4A\log\left(\max(4,\ 2\max(C\sqrt{Z_\star},1))\right)\\
    &=\max\left(4A\log(4),\ 4A\log(2\log(\max(C\sqrt{Z_\star}, 1)))\right)
\end{align*}
Now, if $f(Z_\star)\le 4A\log(4)$ then we clearly have $f(Z_\star)\le Z_\star$ as desired. So, let us focus on the case $f(Z_\star)\le 4A\log(2\log(\max(C\sqrt{Z_\star}, 1)))$. Next, if $C\sqrt{Z_\star} \le e$, then we have $f(Z_\star)\le  4A\log(2)\le Z_\star$ again as desired. Thus we may further restrict to the case $C\sqrt{Z_\star} \ge e$ so that $\max(C\sqrt{Z_\star},1)=C\sqrt{Z_\star}$ and the bound on $f(Z_\star)$ is $4A\log(2\log(C\sqrt{Z_\star}))$. Plugging in our expression for $Z_\star$:
\begin{align*}
    f(Z_\star)&\le 4A\log\left[2\log\left[2C\sqrt{A\log \left[4\log(e+QA)\right]}\right]\right]
\end{align*}

Comparing with the expression for $Z_\star$, we see that to establish $Z_\star \ge f(Z_\star)$, it suffices to show:
\begin{align*}
(e+QA)^2&\ge 2C\sqrt{A\log \left[4\log(e+QA)\right]}
\end{align*}
Now, using $\log(x)\le x$ twice:
\begin{align}
    2C\sqrt{A\log(4\log(e+QA))}&\le 4C\sqrt{A}\sqrt{(e+QA)}\\
    &= \sqrt{16C^2A}\sqrt{e+QA}\\
    &\le \sqrt{e+16C^2A}\sqrt{e+QA}
\intertext{if we set $Q=16C^2$, we will have:}
&= e+QA\\
&\le (e+QA)^2
\end{align}
Thus, by setting $Q=16C^2$, we will have $Z_\star \ge f(Z_\star)$ and so we have $f(Z_\star)\le Z_\star$, which implies 
\begin{align*}
    Z&\le Z_\star\\
    &=4A\log\left[4\log(e+16C^2A)\right]\\
    &\le 4A\log\left(4\sqrt{B} \log(e+16C^2A)\right)
\end{align*}
as desired since $B\ge 1$.
\end{proof}
\end{document}